\documentclass{article}

     \usepackage[final]{neurips_2025}

\usepackage[utf8]{inputenc} 
\usepackage[T1]{fontenc}    
\usepackage{hyperref}       
\usepackage{url}            
\usepackage{booktabs}       
\usepackage{amsfonts}       
\usepackage{nicefrac}       
\usepackage{microtype}      
\usepackage{xcolor}         
\usepackage[table]{xcolor}  
\usepackage{multirow}

\usepackage{amsmath}
\usepackage{amssymb}
\usepackage{mathtools}
\usepackage{amsthm}

\usepackage{mdframed}
\usepackage{xcolor} 
\usepackage[most]{tcolorbox} 

\definecolor{lightgreen}{RGB}{220, 255, 220}

\newmdenv[
  backgroundcolor=green!5,
  linecolor=green!60!black,
  linewidth=0.5pt,
  roundcorner=2pt,
  skipabove=5pt,
  skipbelow=5pt,
  innertopmargin=6pt,
  innerbottommargin=6pt,
  innerleftmargin=6pt,
  innerrightmargin=6pt
]{greenbox}

\newmdenv[
  backgroundcolor=orange!10,
  linecolor=orange!80!black,
  linewidth=0.5pt,
  roundcorner=2pt,
  skipabove=5pt,
  skipbelow=5pt,
  innertopmargin=6pt,
  innerbottommargin=6pt,
  innerleftmargin=6pt,
  innerrightmargin=6pt
]{orangebox}

\newmdenv[
  backgroundcolor=blue!10,
  linecolor=blue!80!black,
  linewidth=0.5pt,
  roundcorner=2pt,
  skipabove=5pt,
  skipbelow=5pt,
  innertopmargin=6pt,
  innerbottommargin=6pt,
  innerleftmargin=6pt,
  innerrightmargin=6pt
]{bluebox}

\usepackage[capitalize,noabbrev]{cleveref}

\theoremstyle{plain}
\newtheorem{theorem}{Theorem}[section]

\newtheorem{lemma}[theorem]{Lemma}

\theoremstyle{definition}
\newtheorem{definition}[theorem]{Definition}

\newtheorem{remark}[theorem]{Remark}

\newcommand\R{\mathbb{R}}

\title{Spectral Conditioning of Attention Improves Transformer Performance}

\author{%
Hemanth Saratchandran \\
Australian Institute for Machine Learning\\
Adelaide University\\
  \texttt{hemanth.saratchandran@adelaide.edu.au} \\
 \And
Simon Lucey \\
Australian Institute for Machine Learning \\
Adelaide University \\
 \texttt{simon.lucey@adelaide.edu.au} \\
}

\begin{document}

\maketitle

\begin{abstract}
We present a theoretical analysis of the Jacobian of a attention block within a transformer, showing that it is governed by the query, key, and value projections that define the attention mechanism. Leveraging this insight, we introduce a method that systematically alters the spectral properties of each attention layer to reduce the Jacobian’s condition number, thereby improving the overall conditioning of the attention layers within a transformer network. We empirically show that this improved Jacobian conditioning translates to enhanced performance in practice. Our approach is simple, broadly applicable, and can be easily integrated as a drop-in replacement for a wide range of existing attention mechanisms. We validate its effectiveness across diverse transformer architectures and tasks, demonstrating consistent improvements in performance.
\end{abstract}

\section{Introduction}
Since its introduction in \cite{vaswani2017attention}, the modern transformer has emerged as one of the most influential architectures in modern machine learning, yielding advances across a wide range of domains, including natural language processing (NLP) \cite{vaswani2017attention, zhuang2021robustly, zhen2022cosformer}, computer vision \cite{dosovitskiy2020image, liu2021swin, touvron2021training, carion2020end}, and robotics \cite{salzmann2020trajectron++, maiti2023transfusion}. At the heart of its success lies the attention mechanism, which enables the model to capture complex relationships by computing pairwise interactions between all input tokens. By dynamically assigning importance to different tokens, attention allows transformers to model global dependencies, making them a foundational architecture in deep learning.

In this work, we investigate the conditioning of the Jacobian associated with the attention layer in transformer models. The conditioning of a matrix is measured by its condition number, the ratio of its largest to smallest singular value, with high values indicating ill-conditioning. Poorly conditioned Jacobians can hinder performance for gradient-based optimizers \cite{nocedal1999numerical}. Recent advances in feedforward neural networks have shown that improving Jacobian conditioning can lead to better optimization and generalization performance. For instance, \cite{saratchandran2025weight} proposed a weight normalization method to address ill-conditioning, while \cite{liu2022loss} demonstrated that reducing the condition number of the neural tangent kernel (NTK) leads to better convergence. Despite these insights, the conditioning of Jacobians in transformer models, particularly within attention layers remains largely unexamined. Addressing this gap is the central focus of our work.

We present a theoretical framework showing that the conditioning of the Jacobian in a attention layer is influenced by the conditioning of the queries, keys, and values matrices that constitute the attention matrix within that layer. Building on this insight, we introduce a method that modifies the spectrum of these matrices by adding a carefully designed correction term, yielding a new attention block that we call \textit{spectral conditioned attention}, resulting in significantly improved Jacobian conditioning. Our analysis shows that the optimal correction can be derived using the singular value decomposition (SVD) of the query, key, and value matrices. However, computing the SVD directly can be prohibitively expensive in large-scale models. To address this, we also propose a computationally efficient approximation that achieves comparable improvements in conditioning with substantially lower overhead. 

While a full theoretical proof linking our methodology to improved convergence in gradient-based optimization would require analyzing the NTK of a transformer, an extremely difficult task, we provide strong empirical evidence that our approach improves transformer performance across a wide range of applications, including image classification, object detection, instance segmentation, and natural language processing. A key advantage of our approach is its compatibility with a broad class of different attention mechanisms \cite{ali2021xcit, liu2021swin, touvron2021training, yuan2022volo, ding2022davit, xiong2021nystromformer}, which we confirm empirically.  Our main contributions include:
\begin{itemize}
    \item[1.] A theoretical framework that analyses the condition number of the Jacobian of attention showing its dependence on the condition number of the query, key and value matrices comprising an attention layer.
    \item[2.] The introduction of spectral conditioned attention that adds a correction term to each of the query, key and value matrices resulting in a better conditioned attention Jacobian.  
\end{itemize}

We validate the above contributions across a range of transformer architectures and applications, including image classification, object detection, instance segmentation, and natural language processing, consistently demonstrating improved performance in each setting.

\section{Related Work}

\paragraph{Conditioning.}
Several recent works have highlighted the critical role of conditioning in the optimization and generalization of neural networks. In \cite{saratchandran2025weight}, the authors focus on the condition numbers of weight matrices in feedforward architectures, showing that well-conditioned weights correlate with improved performance across tasks. They propose a preconditioning strategy that multiplies each weight matrix by a carefully constructed transformation, effectively reducing its condition number and facilitating more stable training. In a complementary line of work, \cite{liu2022loss} investigates neural network optimization from the perspective of the neural tangent kernel (NTK). The authors  demonstrate that improving the conditioning of the NTK leads to better convergence, especially in the infinite-width regime where the NTK governs training dynamics \cite{jacot2018neural}. In \cite{macdonald2023skip}, it is shown that normalization improves the conditioning of various neural network architectures. Similarly, \cite{zhai2023stabilizing} demonstrate that normalizing attention weights enhances training convergence, which can be interpreted as an alternative means of improving the conditioning of attention layers. In a different approach, \cite{agarwal2021deep} shows that simply increasing the depth of feedforward networks can enhance their conditioning, thereby improving the effectiveness of gradient-based optimization methods. Recent work on transformers has examined conditioning from multiple perspectives, including tokenization, activation functions, and optimization \citep{saratchandran2025enhancing, ji2025always, saratchandran2024rethinking, ji2025cutting, saratchandran2025leaner, zheng2025structured}, as well as within the context of fine-tuning \citep{ji2024sine, albert2025randlora, albert2025towards, gordon2025compressing}.


\paragraph{Attention.} A range of methods have been proposed to improve the efficiency and scalability of Transformer architectures, often by rethinking the design of attention mechanisms or reducing computational complexity. For example, the Data-Efficient Image Transformer (DeiT) \cite{touvron2021training} introduces distillation tokens to achieve competitive performance in vision tasks without relying on large-scale datasets. The Cross-Covariance Image Transformer (XCiT) \cite{ali2021xcit} reformulates attention by leveraging cross-covariances of spatial features, enabling efficient spatial interactions with reduced overhead. Swin Transformer \cite{liu2021swin} introduces a hierarchical architecture and a novel shifted window-based attention mechanism, significantly improving efficiency and effectiveness in vision tasks. DaViT \cite{ding2022davit} extends traditional vision transformers by incorporating both spatial and channel attention mechanisms, enhancing feature extraction across multiple dimensions. The Nystr\"{o}mformer \cite{xiong2021nystromformer} takes a different approach by approximating full self-attention using the Nystr\"{o}m method, significantly lowering the quadratic complexity to near-linear while maintaining the core properties of attention, making it especially effective for long-sequence modeling. 

\section{Methodology}\label{sec:theory}

\subsection{Preliminaries}\label{subsec:prelims}
In this section, we define the transformer architecture by describing the structure of a transformer layer, along with establishing notation for various mathematical quantities. 

A transformer layer can be represented as a mapping 
\begin{equation}
  \mathbf{T}: \mathbb{R}^{N \times D} \rightarrow \mathbb{R}^{N \times D}  
\end{equation}
which is formally defined by 
\begin{equation}\label{eqn:trans_main}
    \mathbf{T}(X) = \mathbf{F}(\mathbf{A}(X) + X),
\end{equation}
where \( \mathbf{F} \) denotes a feed forward neural network with a residual connection, and \( \mathbf{A} \) represents an attention mechanism. In general, layer normalization is added however for simplicity we omit layer normalization for this discussion.

For the theoretical framework we will primarily focus on self-attention, which is one of the most common forms of attention. Self-attention is composed of three learnable matrices, query $W_Q \in \R^{D\times d}$, key $W_K \in \R^{D \times d}$, and value 
$W_V \in \R^{D\times d}$, defined for an input sequence $X \in \mathbb{R}^{N \times D}$ by
\begin{equation}\label{eqn:attn_eqn_general}
    \mathbf{A}(X) = \mathbf{softmax}(XW_QW_K^TX^T)XW_V
\end{equation}
where $\mathbf{softmax}$ is the softmax activation that acts row-wise on a matrix \cite{prince2023understanding}. Note that then $A(X) \in \R^{N\times d}$. 
When training a transformer the self-attention parameters are given by the weight matrices $W_Q$, $W_K$ and $W_V$. For further details on transformers readers may consult \cite{prince2023understanding}.


The self-attention map of a layer in a transformer $\mathbf{A}(X)$ has parameters given by those parameters in $X$ from the previous layer and those given by $W_Q$, $W_K$ and $W_V$ that define $\mathbf{A}(X)$. Our work will consider the Jacobian of $\mathbf{A}(X)$ with respect to the parameters within the layer of $\mathbf{A}(X)$, namely $W_Q$, $W_K$ and $W_V$. Therefore, when we speak of the Jacobian of $\mathbf{A}(X)$ it will be with respect to $W_Q$, $W_K$, $W_V$. We will denote this Jacobian by 
$J(\mathbf{A}(X))$ and note that it is defined by
\begin{equation}\label{eqn:jac_defn}
 J(\mathbf{A}(X)) =  
 \bigg{[}\frac{\partial{\mathbf{A}(X)}}{\partial{W_Q}}, 
 \frac{\partial{\mathbf{A}(X)}}{\partial{W_K}}, 
 \frac{\partial{\mathbf{A}(X)}}{\partial{W_V}}\bigg{]}^T  
\end{equation}

Given a matrix $W \in \R^{m\times n}$ we denote the vectorization of $W$ by $\mathbf{vec}(W) \in \R^{mn\times 1}$ \cite{magnus2019matrix}. Note that for such a matrix
there is a transformation $T_{mn} \in \R^{mn \times mn}$ such that
$T_{mn}\mathbf{vec}(W) = \mathbf{vec}(W^T)$ where $W^T$ denotes the transpose of $W$. The matrix $T_{mn}$ is known as a commutation matrix and is a permutation matrix \cite{magnus2019matrix}. The maximum singular value of a matrix $W$ will be denoted by $\sigma_{\max}(W)$ and the minimum singular value by $\sigma_{\min}(W)$. We will use the standard terminology SVD to denote the singular value decomposition of a matrix. Given a vector $z \in \R^n$ the notation $||z||_2$ will denote the vector 2-norm of $z$.

\subsection{Main Theorems}\label{subsec:main_theorems}
In this section, we analyze the conditioning of the Jacobian of the self-attention matrix in a transformer. We will show that the condition number of the Jacobian depends on the condition number of the queries, keys and values matrices. This then motivates our key insight: reducing the condition number of the queries, keys and values matrices can lead to a lower condition number for the Jacobian of the self-attention matrix in each layer of a transformer, which we empirically verify in \cref{sec:exps}. We have chosen to focus our theory on self-attention so as to yield concrete formulas. However, our theory goes through for more general attention mechanisms, see \cref{app:theory} for a discussion.
Proofs of all lemmas and theorems in this section can be found in \cref{app:theory}.

\begin{definition}
Let $A$ be an $N \times d$ matrix of full rank. The condition number of $A$, denoted by $\kappa$, is defined as
\begin{equation}
    \kappa(A) = \frac{\sigma_{\max}(A)}{\sigma_{\min}(A)}
\end{equation}
where $\sigma_{\max}(A)$ denotes the maximum singular value of $A$ and $\sigma_{\min}(A)$ the minimum singular value of $A$, which we know is non-zero as $A$ is full rank.
\end{definition}

Our objective is to analyze the condition number of the self-attention block within a transformer. We demonstrate that the condition number of its Jacobian is influenced by the condition numbers of the query, key, and value weight matrices. By appropriately conditioning these matrices, we show that the Jacobian’s condition number can be significantly reduced, leading to a more stable and well-conditioned attention mechanism suitable for applications.

To establish this, we first examine the derivatives of the self-attention block with respect to the parameters $W_Q$, $W_K$, and $W_V$, yielding \cref{thm:attn_derivs} which is then used to establish a condition number bound on the Jacobian given in \cref{thm:cond_jac}. As a starting point, \cref{lem:softmax_deriv} presents the derivative of the $\mathbf{softmax}$ function, which plays a central role in the analysis.



\begin{lemma}\label{lem:softmax_deriv}
Let $\Lambda : \R^n \rightarrow \R^{n\times n}$ denote the function 
$\Lambda(z) = Diag(z) - z\cdot z^T$.
We then have that 
\begin{equation}
\frac{\partial \mathbf{softmax}}{\partial{x}}(z) = 
\Lambda(\mathbf{softmax}(z)).
\end{equation}
\end{lemma}

\begin{greenbox}
\begin{theorem}\label{thm:attn_derivs}
Let $\mathbf{A}(X)$ denote a self-attention matrix with input $X$ as defined by \cref{eqn:attn_eqn_general}. Then
\begin{align}
 \frac{\partial{\mathbf{A}(X)}}{\partial{W_Q}} &= 
 (W_V^TX^T\otimes I_{N})\bigg{(} 
 \Lambda(\mathbf{softmax}(XW_QW_K^TX^T))
 \bigg{)}(XW_K\otimes X) \\
\frac{\partial{\mathbf{A}(X)}}{\partial{W_K}} &= 
(W_V^TX^T\otimes I_N)\bigg{(} 
 \Lambda(\mathbf{softmax}(XW_QW_K^TX^T))
 \bigg{)}(X\otimes XW_Q)\cdot T_{Dd} \\
\frac{\partial{\mathbf{A}(X)}}{\partial{W_V}} &= 
I_d\otimes \mathbf{softmax}(XW_QW_K^TX^T)X
\end{align}
where $\Lambda$ is defined in \cref{lem:softmax_deriv} and $T_{Dd}$ is the commutation matrix satisfying 
$T_{Dd}\mathbf{vec}(W_K) = \mathbf{vec}(W_K^T)$ (see \cref{subsec:prelims}).
\end{theorem}
\end{greenbox}

\begin{greenbox}
\begin{theorem}\label{thm:cond_jac}
Let $\mathbf{A}(X)$ denote a self-attention matrix, as defined in 
\cref{eqn:attn_eqn_general}, with input $X$ and let 
$J(\mathbf{A}(X))$ denote its Jacobian with respect to the parameter matrices $W_Q$, $W_K$ and $W_V$ as defined in 
\cref{eqn:jac_defn}. Assume that $J(\mathbf{A}(X))$ has full rank so that $\kappa(J(\mathbf{A}(X)))$ is finite.
Then 
\begin{align}
  \kappa(J(\mathbf{A}(X))) \leq 
\kappa(X)^3
&\kappa\big{(}\Lambda(\mathbf{softmax}(XW_QW_K^TX^T))\big{)}\kappa(W_V)\big{(} 
\kappa(W_Q) + \kappa(W_K)
\big{)} \label{eqn:cond_jac}\\
&+ \kappa(X)\kappa(\mathbf{softmax}(XW_QW_K^TX^T)) \nonumber
\end{align}
where $\Lambda$ is defined in \cref{lem:softmax_deriv}.
\end{theorem}
\end{greenbox}

\cref{thm:cond_jac} shows that
$\kappa(J(\mathbf{A}(X)))$ is bounded above by a sum of two terms:
\begin{align}
&\kappa(X)^3
\cdot\kappa\big{(}\Lambda(\mathbf{softmax}(XW_QW_K^TX^T))\big{)}\kappa(W_V)\big{(} 
\kappa(W_Q) + \kappa(W_K)
\big{)} \label{first_term}\\
&\kappa\big(\mathbf{softmax}(XW_QW_K^\top X^\top)\big) \label{second_term}
\end{align}

\textbf{Observation:} Directly conditioning the Jacobian of the self-attention map with respect to the parameters $W_Q$, $W_K$, and $W_V$ during training would require computing the Jacobian at every iteration, which is prohibitively expensive. However, \cref{thm:cond_jac} offers a more efficient alternative: reducing the upper bound in \cref{eqn:cond_jac}. Since we have access to $W_Q$, $W_K$, and $W_V$, lowering their condition numbers reduces the quantity in \cref{first_term}, thereby tightening the bound in \cref{eqn:cond_jac}. The following theorem shows how to do this.

\begin{greenbox}
\begin{theorem}\label{thm:attention_weights_regularize}
Let $W_Q$, $W_K$ and $W_V$ denote the parameters of a self-attention matrix $\mathbf{A}(X)$.
Then there exist matrices $C_Q$, $C_K$ and $C_V$ such that 
\begin{equation}
    \kappa(W_Q + C_Q)\text{, } \kappa(W_K + C_K)\text{, } \kappa(W_V + C_V) \leq 2
\end{equation}
\end{theorem}
\end{greenbox}


\begin{figure}[ht!]
    \centering
    \includegraphics[width=1.\linewidth]
    {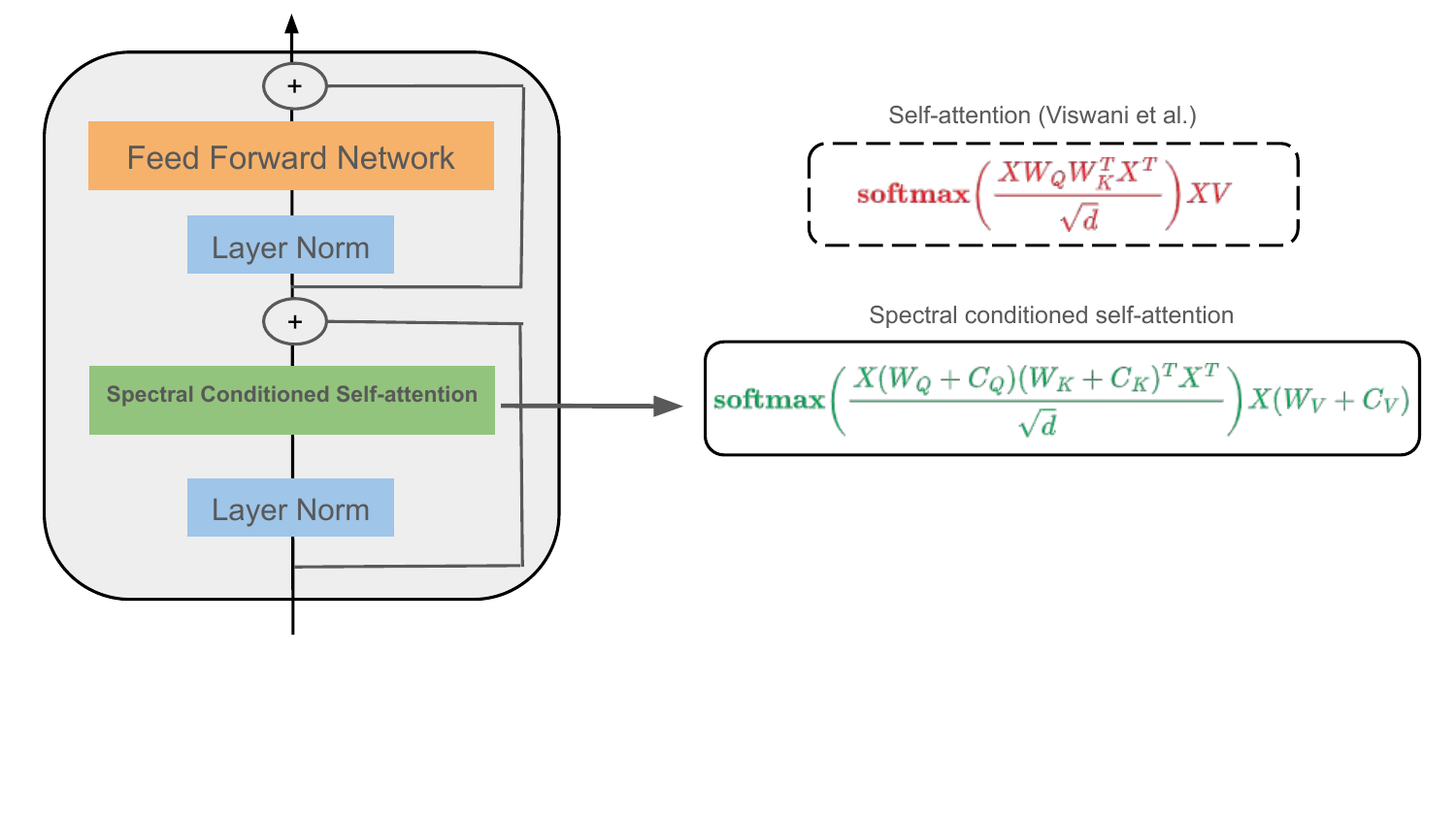}
    \vspace{-4em}
    \caption{An illustration of spectrally conditioned self-attention within a transformer layer. At each layer, the self-attention weights $W_Q$, $W_K$, and $W_V$ are modified by adding correction terms $C_Q$, $C_K$, and $C_V$, respectively. The correction terms $C_Q$, $C_K$, and $C_V$ are initialized before training using \cref{thm:imp_friendly} and remain fixed throughout training.}
    \label{fig:spec_cond_diagram}
\end{figure}

\begin{orangebox}
\paragraph{Overview of proof of \cref{thm:attention_weights_regularize}:} The same proof works for 
$W_Q$, $W_K$ and $W_V$. Thus we will give the main steps for $W_Q$. As in \cref{subsec:prelims} let $W_Q \in \R^{D\times d}$.
\begin{itemize}
    \item[1.] Using the SVD write $W_Q = USV^T$ where $U \in \R^{D\times D}$ is the left singular vectors, 
    $S \in \R^{D \times d}$ is the matrix of singular values on the main diagonal and $V \in \R^{d \times d}$ are the right singular vectors. 
    \item[2.] Define a new matrix $C_Q = U\overline{S}V^T$ where $\overline{S}$ has $\sigma_{\max}(W_Q)$ on its main diagonal and zeros everywhere else.
    \item[3.] The matrix $W_Q + C_Q$ then satisfies 
    $\kappa(W_Q + C_Q) \leq 2$.
\end{itemize}
\end{orangebox}

\subsection{Spectral Conditioned Attention}\label{subsec:spec_cond}

In \cref{thm:attention_weights_regularize}, we showed that adding correction terms $C_Q$, $C_K$, and $C_V$ to the weight matrices $W_Q$, $W_K$, and $W_V$, respectively, significantly reduces their condition numbers $\kappa(W_Q)$, $\kappa(W_K)$, and $\kappa(W_V)$.
This motivates the following definition. 

\begin{definition}\label{defn:spec_condition}
We define \textit{spectral conditioned attention} by 
\begin{equation}
    \mathbf{SpecA}(X) := 
    \mathbf{softmax}(X(W_Q+C_Q)(W_K + C_K)^TX^T)X(W_V + C_V)
\end{equation}
where $C_Q$, $C_K$ and $C_V$ are the correction terms given in 
\cref{thm:attention_weights_regularize}.
\end{definition}
We call the process of adding correction matrices $C_Q$, $C_K$ and $C_V$ to $W_Q$, $W_K$ and $W_V$ respectively to lower their condition number as spectral conditioning. 

\begin{remark}
\cref{thm:attention_weights_regularize}, together with \cref{thm:cond_jac}, shows that spectrally conditioned attention reduces the upper bound in \cref{eqn:cond_jac}, thereby yielding a tighter bound on the condition number of the Jacobian of the self-attention layer.
\end{remark}

The proof overview for \cref{thm:attention_weights_regularize} showed the correction term $C_Q$ (and $C_K$, $C_V$) is derived by taking an SVD. Computing the SVD of each of $W_Q$, $W_K$, $W_V$ at each iteration of training would lead to a significant memory overhead. We therefore provide an implementation friendly form of spectral conditioned attention, motivated by the following theorem (proof given in \cref{app:theory}).
\begin{greenbox}
\begin{theorem}\label{thm:imp_friendly}
Let $A \in \R^{m \times n}$ and  
let $I_k \in \R^{m\times n}$ be the matrix that has $1$ on its main $k \times k$ diagonal, where $k = \min\{m, n\}$, and zeros elsewhere.
Let $\lambda \geq 2$ be a fixed constant and assume that
$\frac{\sigma_{\max}(A)+ \lambda}{\lambda - \sigma_{\min}(A)} \leq \frac{\sigma_{\max}(A)}{\sigma_{\min}(A)}$. Then
$\kappa(A + \lambda I_k) < \kappa(A)$.
\end{theorem}
\end{greenbox}
The key difference between \cref{thm:attention_weights_regularize} and \cref{thm:imp_friendly} when applied to $W_Q$, $W_K$, and $W_V$ is that \cref{thm:attention_weights_regularize} guarantees a condition number strictly less than $2$, whereas \cref{thm:imp_friendly} only ensures a reduction relative to the original i.e. $\kappa(W_Q + \lambda I_k) < \kappa(W_Q)$, and similarly for $W_K$ and $W_V$. However, the advantage of \cref{thm:imp_friendly} over \cref{thm:attention_weights_regularize} is that the correction term $\lambda I_k$ does not require computing the SVD of $W_Q$, $W_K$ or $W_V$ and hence is much more memory efficient.

\begin{bluebox}
\paragraph{Implementation:} We initialized correction matrices $C_Q$, $C_K$, and $C_V$ as $\lambda I_k$ with $\lambda = 10$ (see \cref{app:ic} for an ablation study on $\lambda$), and used spectral conditioned attention where the modified weights were $W_Q + C_Q$, $W_K + C_K$, and $W_V + C_V$. These correction matrices are fixed during training and are not updated, incurring no additional memory overhead during backpropagation. An overview of the resulting attention architecture is shown in \cref{fig:spec_cond_diagram}.
\end{bluebox}
For a comparison of spectral conditioned attention using \cref{thm:attention_weights_regularize} and \cref{thm:imp_friendly} on vision transformers see \cref{app:ic}.

\section{Experiments}\label{sec:exps}

In this section, we evaluate our insights from \cref{sec:theory} on a variety of transformer applications. For each application we will consider an original transformer architecture used within the literature and compare it with one employing spectral conditioning on its attention layer. 
\paragraph{Implementation.}
In all cases, we used the implementation described in \cref{subsec:spec_cond}, where fixed correction terms $C_Q$, $C_K$, and $C_V$ are added to the query, key and value matrices $W_Q$, $W_K$ and $W_V$ respectively within each attention layer, with $\lambda = 10$ (see \cref{app:ic} for an ablation on $\lambda$). These terms are not updated during training and therefore do not introduce any additional trainable parameters. For more details see \cref{app:exps}.

\subsection{Image Classification}\label{subsec:IC}

\paragraph{Vision transformers.} We applied spectral conditioned attention to a variety of modern vision transformers: the original vision transformer (ViT-B) \cite{dosovitskiy2020image}, Swin transformer (Swin-B) \cite{liu2021swin},  Cross-Covariance image transformer (XCiT-M) \cite{ali2021xcit}, Data efficient image transformer (DeiT-B) \cite{touvron2021training}, Dual attention vision transformer (DaViT-B) \cite{ding2022davit} for image classification on ImageNet-1k. Each of these vision transformers uses a different attention block to the standard self-attention used in the ViT-B architecture. However, spectral conditioned attention (see \cref{defn:spec_condition}) works by adding a correction term to each query, key and value matrix and hence can easily be plugged into these modern variants of self-attention. We include these latter models to illustrate that, while our theoretical analysis in \cref{sec:theory} centers on self-attention, spectral conditioned attention can be readily incorporated into modern transformer architectures, including those with more complex designs. 

\paragraph{Validating the theory.}
We validate the theoretical results from \cref{sec:theory} on a ViT-B model trained on the ImageNet-1k dataset. For each experiments we ran five trials and took the mean and standard deviations.
For each attention head in every transformer layer, we compute the minimum singular value of the query ($W_Q$), key ($W_K$), and value ($W_V$) weight matrices, and then average these values across all heads and layers.
To improve the conditioning, we add correction terms $C_Q$, $C_K$, and $C_V$ as described after \cref{thm:imp_friendly}, resulting in spectral conditioned attention matrices $W_Q + C_Q$, $W_K + C_K$, and $W_V + C_V$. The left figure in \cref{fig:vit_analysis} shows the average minimum singular value during training and in each case we see that the minimum singular value of the corrected weights is higher. 
Similar plots for the maximum singular value is shown in \cref{app:ic}.
The middle plot in \cref{fig:vit_analysis} compares the average condition numbers of the original and spectral conditioned query, key and value matrices, demonstrating that the latter are substantially better conditioned, thus empirically verifying the claim of \cref{thm:imp_friendly}.
Finally, the right plot reports the average condition number of the Jacobian of the self-attention matrices, both for the standard ViT-B and for the spectral conditioned version. It also includes the theoretical upper bound from \cref{thm:cond_jac}. While the spectral correction was motivated by this theoretical bound, the results clearly show that it leads to better-conditioned Jacobians in practice, supporting its use in designing more stable attention mechanisms. Note that each plot shows the mean over five trials. Standard deviations cannot be seen due to the log plot of the y-axis but can be found in separate plots in \cref{app:ic}.
We repeated the experiment using the XCiT-M transformer \cite{ali2021xcit}. As shown in \cref{fig:xcit_analysis}, the results exhibit the same trend observed with ViT-B, further supporting the theoretical findings in \cref{sec:theory}. 

\begin{figure}[ht!]
    \centering
    \includegraphics[width=1.\linewidth]
    {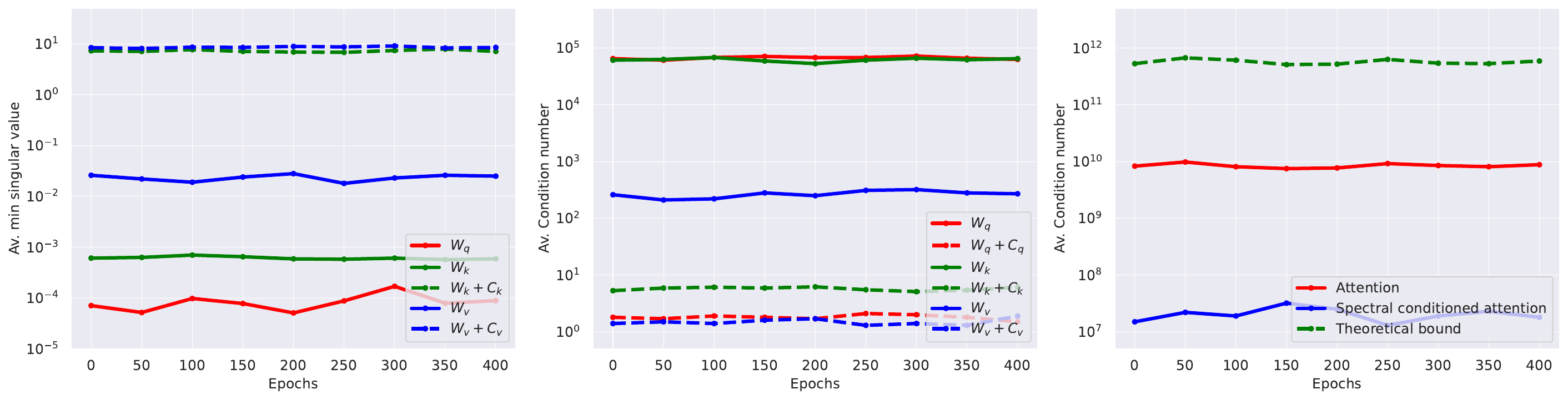}
    \caption{Analysis for ViT-B. \textbf{Left:} Average minimum singular value of the query, key, and value projection matrices ($W_Q$, $W_K$, $W_V$) and their spectrally conditioned counterparts ($W_Q + C_Q$, $W_K + C_K$, $W_V + C_V$) throughout training. \textbf{Middle:} Condition numbers of $W_Q$, $W_K$, and $W_V$, and their spectrally conditioned forms during training. \textbf{Right:} Average condition number of the self-attention Jacobian over the course of training, before and after spectral conditioning, along with the theoretical bound from \cref{eqn:cond_jac}.}
    \label{fig:vit_analysis}
\end{figure}

\begin{figure}[ht!]
    \centering
    \includegraphics[width=1.\linewidth]
    {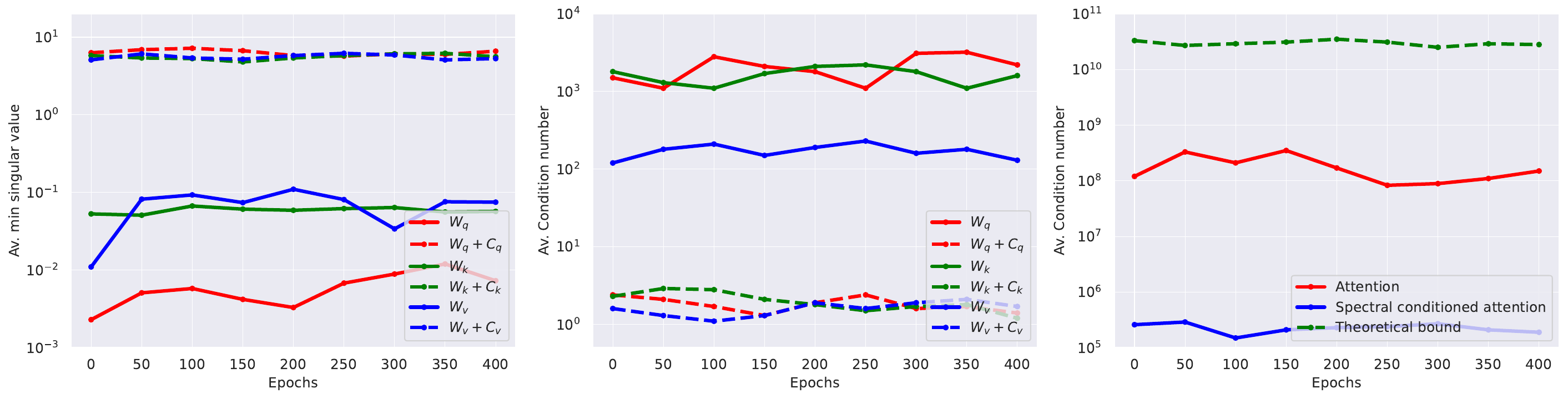}
    \caption{Analysis for XCiT-M. \textbf{Left:} Average minimum singular value of the query, key, and value projection matrices ($W_Q$, $W_K$, $W_V$) and their spectrally conditioned counterparts ($W_Q + C_Q$, $W_K + C_K$, $W_V + C_V$) throughout training. \textbf{Middle:} Condition numbers of $W_Q$, $W_K$, and $W_V$, and their spectrally conditioned forms during training. \textbf{Right:} Average condition number of the self-attention Jacobian over the course of training, before and after spectral conditioning, along with the theoretical bound from \cref{eqn:cond_jac}.}
    \label{fig:xcit_analysis}
\end{figure}

\paragraph{Rank assumption.} In the statement of \cref{thm:cond_jac}, we assumed that the Jacobian is full rank, ensuring the condition number in \cref{eqn:cond_jac} remains finite. Empirically, we observed this assumption holds across all vision transformer architectures considered. For both ViT-B and XCiT-M, the final plots in \cref{fig:vit_analysis} and \cref{fig:xcit_analysis} confirm that the average condition number remains finite and stable throughout training, indicating no rank deficiency in the Jacobian of each layer.

\paragraph{Results on ImageNet-1k.} We trained two versions of each transformer: a baseline model and one incorporating spectral conditioned attention. For each base architecture, we observed that the minimum singular values of the query, key, and value matrices remained below $1$ throughout training (see \cref{fig:vit_analysis} for ViT-B and \cref{fig:xcit_analysis} for XCiT-M), allowing us to apply \cref{thm:imp_friendly} and implement spectral conditioning as described in \cref{subsec:spec_cond}. All models were trained using the AdamW optimizer, see \cref{app:ic}. We trained each model five times with five different random seeds.
The final results, summarized in \cref{tab:vits}, show the mean test accuracy with the standard deviation in brackets. We observe that in every case, the spectral conditioned variant achieves higher test accuracy than its unmodified counterpart. Hardware, memory cost and FLOPS analysis is shown in \cref{app:ic}. An ablation on $\lambda$ is given in \cref{app:nyst}. Furthermore, spectral conditioning was applied together with layer normalization as shown in \cref{fig:spec_cond_diagram}. A discussion comparing the two is given in \cref{app:nyst}.


\begin{table*}[!ht]
\caption{Comparison of Vision Transformers with their original attention architecture versus one with spectrally conditioned attention (Spec. cond.), pre-trained on the ImageNet-1k dataset. We report Top-1\% classification accuracy. In each case, spectral conditioning improves performance over the original attention mechanism.}
\vspace{0.2cm}
\label{tab:vits}
\centering
\begin{tabular}{c|c c c c c}
    \toprule
    \rowcolor{gray!10}
    & ViT-B & DeiT-B & Swin-B & XCiT-M & DaViT-B \\
    \midrule
    Original & 80.7 ($\pm$0.41) & 81.6 ($\pm$0.30) & 83.4 ($\pm$0.28) & 82.6 ($\pm$0.39) & 84.3 ($\pm$0.26) \\
    \rowcolor{lightgreen}
    Spec. cond. & 81.7 ($\pm$0.38) & 82.6 ($\pm$0.32) & 84.1 ($\pm$0.25) & 83.5 ($\pm$0.35) & 84.9 ($\pm$0.21) \\
    \bottomrule
\end{tabular}
\end{table*}

\subsection{Object Detection and Instance Segmentation}\label{subsec:OD}
In this section, we evaluate our methodology on two downstream tasks: object detection and instance segmentation. The goal is to assess the effectiveness of spectral conditioned attention in a fine-tuning setting. We begin by pretraining an XCiT architecture \cite{ali2021xcit} on the ImageNet-1k dataset, followed by fine-tuning on the COCO 2017 dataset \cite{lin2014microsoft}.
The XCiT models are used as backbones within the Mask R-CNN framework \citep{he2017mask}, enhanced with a Feature Pyramid Network (FPN) to improve multi-scale feature representation. To integrate XCiT with FPN, we modify its column structure to extract intermediate features from multiple layers. Specifically, we use the 12-layer XCiT-Small (XCiT-S) adapting their its strides from the original fixed stride of 16 to $[4, 8, 16, 32]$ to match the FPN hierarchy. Downsampling is performed via max pooling, and upsampling is achieved using a single transposed convolution layer.
We conduct this procedure for both the original XCiT-S and one incorporating spectral conditioned attention blocks.


\paragraph{Downstream results.} The results for both object detection and instance segmentation are shown in \cref{tab:transferlearning}. We ran five trials each with a different seed and plotted the mean metric with standard deviation in brackets.
The reported metrics include \( AP^b \) (Average Precision for bounding box predictions), \( AP^b_{50/75} \) (Average Precision at IoU thresholds of 0.50 and 0.75 for bounding boxes), \( AP^m \) (Average Precision for mask predictions), and \( AP^m_{50/75} \) (Average Precision at IoU thresholds of 0.50 and 0.75 for mask predictions). In each case we see the XCiT employing spectral conditioned (spec. cond.) out performs the original architecture. Hardware and memory overheads are discussed in \cref{app:od}.




\begin{table*}[!ht]
\caption{Performance evaluation of object detection and instance segmentation on the COCO dataset. For each metric, our spectrally conditioned architecture (Spec. cond.) outperforms the original.}
\vspace{0.2cm}
\label{tab:transferlearning}
\centering
\setlength{\tabcolsep}{2pt}

\begin{tabular}{cc c c c c c}
    \toprule
    \rowcolor{gray!10}
    Model & $AP^b$ & $AP^b_{50}$ & $AP^b_{75}$ & $AP^m$ & $AP^m_{50}$ & $AP^m_{75}$ \\
    \midrule
    original & 44.9 ($\pm$0.33) & 66.1 ($\pm$0.31) & 48.9 ($\pm$0.35) & 40.1 ($\pm$0.29) & 63.1 ($\pm$0.31) & 42.8 ($\pm$0.34) \\
    \rowcolor{lightgreen}
    Spec. cond. & 45.6 ($\pm$0.31) & 66.7 ($\pm$0.36) & 49.6 ($\pm$0.32) & 40.5 ($\pm$0.33) & 63.4 ($\pm$0.28) & 43.3 ($\pm$0.32) \\
    \bottomrule
\end{tabular}
\end{table*}

\subsection{Nystr\"{o}mformer on LRA Benchmark}\label{subsec:nyst}

To evaluate the effectiveness of our method, we applied it to the Nystr\"{o}mformer architecture \cite{xiong2021nystromformer}, which is specifically designed to handle long-range dependencies efficiently. Experiments were conducted on the Long-Range Arena (LRA) benchmark \cite{tay2020long}, a suite of tasks targeting the ability of models to process extended input sequences. We trained two variants of the Nystr\"{o}mformer: one baseline and one with spectral conditioned attention.

\paragraph{Validating the theory.} 
We validate the theoretical results from \cref{sec:theory} on a Nystr\"{o}mformer trained on the LRA text classification task. Each experiment was repeated five times, reporting the mean and standard deviation. We computed the minimum singular values and condition numbers of $W_Q$, $W_K$, $W_V$, and their spectral-conditioned forms $W_Q + C_Q$, $W_K + C_K$, and $W_V + C_V$. We also evaluated the average condition number of the attention Jacobian, its conditioned variant, and the theoretical bound from \cref{eqn:cond_jac}. The observed trends closely follow those for ViT-B and XCiT-M in \cref{subsec:IC}, supporting our theoretical findings. Results are shown in \cref{fig:nyst_analysis_text}; standard deviations and maximum singular value plots are provided in \cref{app:ic}.


\begin{figure}[ht!]
    \centering
    \includegraphics[width=1.\linewidth]
    {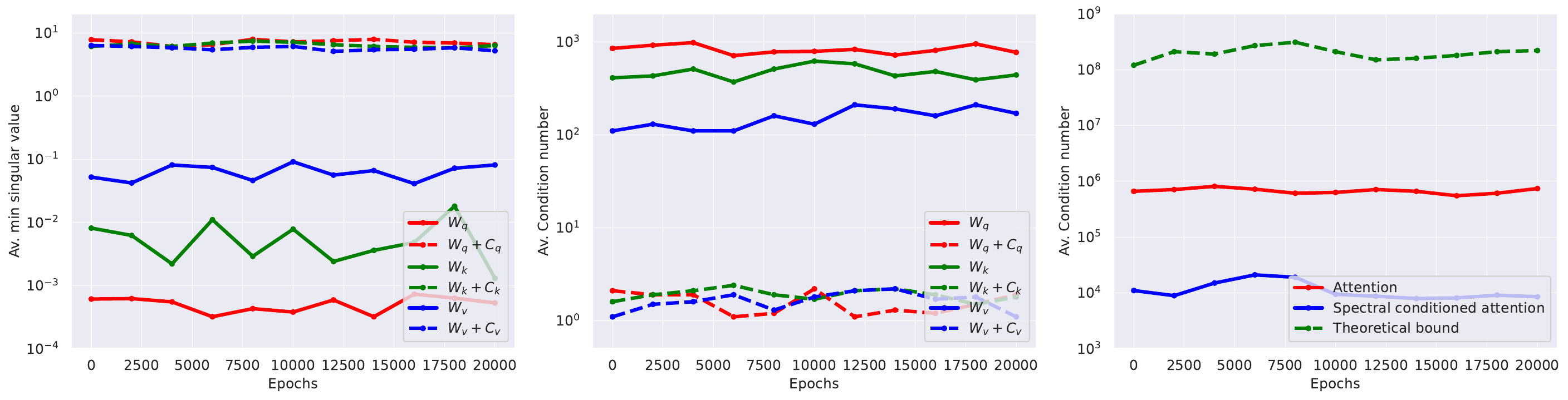}
    \caption{Analysis for Nystr\"{o}mformer on text classification task. \textbf{Left:} Average minimum singular value of the query, key, and value projection matrices ($W_Q$, $W_K$, $W_V$) and their spectrally conditioned counterparts ($W_Q + C_Q$, $W_K + C_K$, $W_V + C_V$) throughout training. \textbf{Middle:} Condition numbers of $W_Q$, $W_K$, and $W_V$, and their spectrally conditioned forms during training. \textbf{Right:} Average condition number of the attention Jacobian over the course of training, before and after spectral conditioning, along with the theoretical bound from \cref{eqn:cond_jac}.}
    \label{fig:nyst_analysis_text}
\end{figure}

\paragraph{Results.} \cref{tab:nystrom} shows the results of the experiment. As can be seen from the table the Nystr\"omformer that employed spectral conditioned attention outperformed the original Nystr\"omformer from \cite{xiong2021nystromformer} in every task in the LRA benchmark. We used the exact training setup and hyperparameters from \cite{xiong2021nystromformer}. Hardware and memory overheads are discussed in \cref{app:nyst} and an ablation on $\lambda$ is given in \cref{app:nyst}. A discussion comparing spectral conditioning and layer normalization is given in \cref{app:nyst}.

\begin{table*}[!ht]
\caption{Comparison of a Nyströmformer using its original architecture (original) with a Nyströmformer using spectral conditioned attention (Spec. cond.) on the LRA benchmark. We report evaluation accuracy (\%). As shown, our methodology consistently improves performance across all tasks.}
\vspace{0.2cm}
\label{tab:nystrom}
\centering
\setlength{\tabcolsep}{7pt}

\begin{tabular}{cc c c c c}
    \toprule
    \rowcolor{gray!10}
    Model & ListOps & Text & Retrieval & Image & Pathfinder \\
    \midrule
    Original & 37.1 ($\pm$0.21) & 63.8 ($\pm$0.24) & 79.8 ($\pm$0.26) & 39.9 ($\pm$0.20) & 72.9 ($\pm$0.25) \\
    \rowcolor{lightgreen}
    Spec. cond. & 37.8 ($\pm$0.23) & 64.8 ($\pm$0.20) & 80.6 ($\pm$0.27) & 40.2 ($\pm$0.22) & 73.7 ($\pm$0.23) \\
    \bottomrule
\end{tabular}
\end{table*}

\subsection{Language Modeling with Crammed BERT}\label{subsec:lang}

We applied our insights from \cref{sec:theory} to a Crammed BERT language model, trained entirely from scratch using masked language modeling following the approach in \cite{geiping2023cramming}. The model consists of 110 million parameters, with 12 transformer layers and 12 attention heads per layer. For this experiment, we train two versions from scratch: the original Crammed BERT baseline model from \cite{geiping2023cramming} and a variant incorporating spectral conditioned attention in each layer, following the implementation from \cref{subsec:spec_cond}. Both models are trained on The Pile dataset \cite{gao2021pile}, a large-scale corpus designed for language model training, following the pretraining regime carried out in \cite{geiping2023cramming}. 
After pretraining, we evaluate the performance on the GLUE benchmark \cite{wang2018glue} following the evaluation methodology outlined in \cite{geiping2023cramming}. Each model was trained five times on the 
Pile dataset \cite{gao2021pile} with five different random seeds. \cref{tab:bert_glue_results} shows the mean results with standard deviations shown in brackets. As can be seen from the table, spectral conditioned attention outperforms the baseline model in every downstream task in the GLUE benchmark. For the
Hardware and memory overheads we refer the reader to \cref{app:cram}.

\begin{table}[h]
\scriptsize
\setlength{\tabcolsep}{2.5pt}  
\caption{Evaluation of a pre-trained Crammed BERT on the GLUE benchmark using the original architecture (Original) and our spectrally conditioned variant (Spec. cond.). Spectral conditioning improves performance across all tasks.}
\label{tab:bert_glue_results}
\centering
\begin{tabular}{lccccccccc}
    \toprule
    \rowcolor{gray!10}
    & MNLI & SST-2 & STSB & RTE & QNLI & QQP & MRPC & CoLA & Avg. \\
    \midrule
    Original & 83.8 ($\pm$0.21) & 92.3 ($\pm$0.18) & 86.3 ($\pm$0.21) & 55.1 ($\pm$0.20) & 90.1 ($\pm$0.23) & 87.3 ($\pm$0.15) & 85.0 ($\pm$0.25) & 48.9 ($\pm$0.26) & 78.6  \\
    \rowcolor{lightgreen}
    Spec. cond. & 84.3 ($\pm$0.20) & 92.7 ($\pm$0.19) & 86.6 ($\pm$0.22) & 55.5 ($\pm$0.24) & 91.0 ($\pm$0.22) & 87.5 ($\pm$0.23) & 86.1 ($\pm$0.23) & 51.7 ($\pm$0.24) & 79.4 \\
    \bottomrule
\end{tabular}
\end{table}

\section{Limitations}\label{sec:limits}
Our spectral conditioning approach is derived by optimizing an upper bound on the condition number of the self-attention Jacobian, rather than directly minimizing the Jacobian’s condition number itself. While this bound provides useful theoretical guidance and aligns with empirical improvements, it remains an indirect proxy. Developing methods to efficiently estimate and control the exact Jacobian conditioning during training would be a valuable direction for future work. Additionally, due to computational resource constraints, our experiments were limited to models with up to 100 million parameters. Whether spectral conditioning offers similar benefits for large-scale transformer models with billions of parameters remains an open question.

\section{Conclusion}\label{sec:conc}
We presented a theoretical framework linking the conditioning of self-attention Jacobians to the spectral properties of the query, key, and value matrices in transformer architectures. Building on this insight, we introduced spectral conditioned self-attention, a simple and broadly applicable method that modifies these matrices with carefully designed correction terms to improve Jacobian conditioning. Our empirical results demonstrate that this approach consistently yields good performance across diverse transformer models and tasks, including image classification, object detection, language modeling, and long-range sequence learning.


\newpage
{
    \bibliographystyle{plain}
    \bibliography{main}
}

\newpage
\appendix

\section{Appendix}

\subsection{Theoretical Framework}\label{app:theory}

In this section we will give the proofs of the lemma and theorems in \cref{subsec:main_theorems} and \cref{subsec:spec_cond}.

To begin with we will need the following standard facts on derivatives of matrices.

\begin{lemma}\label{lem:matrix_deriv}
Let $A \in R^{n\times m}$, $B \in \R^{k\times l}$ and 
$C \in \R^{m\times k}$. Then 
\begin{equation}
    \frac{\partial{ACB}}{\partial{C}} = B^T \otimes A.
\end{equation}
\end{lemma}
\begin{proof}
We start by using a well known vectorization identity \cite{magnus2019matrix}
\begin{equation}\label{eqn:vec_identity}
    \mathbf{vec}(ACB) = (B^T \otimes A)\mathbf{vec}(C)
\end{equation}
where $\mathbf{vec}$ denotes the vectorization operator which takes a matrix and maps it to a vector by stacking its columns on top of each other, see \cite{magnus2019matrix}. We then differentiate \cref{eqn:vec_identity} to obtain
\begin{equation}
\frac{\partial{\mathbf{vec}(ACB)}}{\partial{\mathbf{vec}(C)}} = 
B^T \otimes A.
\end{equation}
The result of the lemma follows.
\end{proof}

\begin{lemma}\label{lem:transpose_deriv}
Let $A \in \R^{n\times m}$ so that $A^T \in \R^{m\times n}$. Then
\begin{align}
    \mathbf{vec}(A^T) &= T_{mn}\mathbf{vec}(A) \\
    \frac{\partial{\mathbf{vec}(A^T)}}{\partial{\mathbf{vec}(A)}} &= T_{mn}
\end{align}
where $T_{mn}$ is a commutation matrix.
\end{lemma}
\begin{proof}
The first equation follows from the definition of the transpose of a matrix \cite{magnus2019matrix}. The second equation then follows from the first.
\end{proof}

We now give the proof of \cref{lem:softmax_deriv}.

\begin{proof}[Proof of \cref{lem:softmax_deriv}]
Let $[z_1,\ldots ,z_n]$ denote a row vector in $\R^n$. Then by definition 
\begin{equation}
\mathbf{softmax}([z_1,\ldots ,z_n]) = 
\bigg{[}\frac{e^{z_1}}{\sum_{i=1}^ne^{z_i}},\ldots ,\frac{e^{z_n}}{\sum_{i=1}^ne^{z_i}}\bigg{]}. 
\end{equation}
From the above equation we can compute the partial derivative and find
\begin{equation}
\frac{\partial \, \mathrm{softmax}(z)_i}{\partial z_j} = \mathrm{softmax}(z)_i \left( \delta_{ij} - \mathrm{softmax}(z)_j \right).
\end{equation}
The term $\frac{\partial \, \mathrm{softmax}(z)_i}{\partial z_j}$ is precisely the $ij$ component of the matrix 
$\frac{\partial \mathbf{softmax}(z)}{\partial z}$. Putting each of these $ij$ terms into an $n \times n$ matrix we find
\begin{equation}
\frac{\partial \mathbf{softmax}(z)}{\partial z} = 
Diag(z) - z\cdot z^T
\end{equation}
which proves the lemma.
\end{proof}

We can use the above lemmas to give the proof of \cref{thm:attn_derivs}.

\begin{proof}[Proof of \cref{thm:attn_derivs}]
We start by establishing the derivative formula for the term 
$\frac{\partial\mathbf{A}(X)}{\partial W_Q}$. We will use the notation used in \cref{subsec:prelims}. Note that by definition
$\frac{\partial\mathbf{A}(X)}{\partial W_Q} \in \R^{dN\times dD}$.
Using \cref{eqn:attn_eqn_general}
\begin{equation}\label{app:eqn:attn_eqn_general}
    \mathbf{A}(X) = I_N\mathbf{softmax}(XW_QW_K^TX^T)XW_V
\end{equation}
where $I_N$ is the $N \times N$ identity matrix. This is done so that we can apply \cref{lem:matrix_deriv}. We then compute
\begin{align}
\frac{\partial\mathbf{A}(X)}{\partial W_Q} &= 
\frac{\partial (I_N\mathbf{softmax}(XW_QW_K^TX^T)XW_V)}{\partial W_Q} \\
&= (W_V^TX^T\otimes I_N)\frac{\partial \mathbf{softmax}(XW_QW_K^TX^T)}{\partial W_Q} \text{ using } \cref{lem:matrix_deriv} \\
&=  (W_V^TX^T\otimes I_N) \Lambda(\mathbf{softmax}(XW_QW_K^TX^T))\frac{\partial (XW_QW_K^TX^T)}{\partial W_Q} 
\text{ using } \cref{lem:softmax_deriv} \\
&= (W_V^TX^T\otimes I_N) \Lambda(\mathbf{softmax}(XW_QW_K^TX^T))
(XW_K\otimes X)
\end{align}
which proves the firs equality in \cref{thm:attn_derivs}. 

To compute $\frac{\partial \mathbf{A}(X)}{\partial W_K} \in \R^{dN\times dD}$ we proceed in a similar way. 
\begin{align}
\frac{\partial \mathbf{A}(X)}{\partial W_K} &= 
\frac{\partial (I_N\mathbf{softmax}(XW_QW_K^TX^T)XW_V)}{\partial W_K} \\
&= 
(W_V^TX^T\otimes I_N)\frac{\partial \mathbf{softmax}(XW_QW_K^TX^T)}{\partial W_K} \text{ using } \cref{lem:matrix_deriv} \\
&=
(W_V^TX^T\otimes I_N) \Lambda(\mathbf{softmax}(XW_QW_K^TX^T))\frac{\partial (XW_QW_K^TX^T)}{\partial W_K} 
\text{ using } \cref{lem:softmax_deriv} \\
&=
(W_V^TX^T\otimes I_N) \Lambda(\mathbf{softmax}(XW_QW_K^TX^T))
(X\otimes XW_Q)T_{Dd} 
\end{align}
where the last equality follows from \cref{lem:matrix_deriv,lem:transpose_deriv}
This establishes the second equality in \cref{thm:attn_derivs}.

To prove the identity for $\frac{\partial\mathbf{A}(X)}{\partial W_V} \in \R^{dN\times dD}$ we write
\begin{equation}
\mathbf{A}(X) = \mathbf{softmax}(XW_QW_K^TX^T)XW_VI_d    
\end{equation}
where $I_d$ is the $d \times d$ identity matrix. Then we
simply apply \cref{lem:matrix_deriv} to obtain
\begin{align}
\frac{\partial\mathbf{A}(X)}{\partial W_V} &= 
\frac{\partial (\mathbf{softmax}(XW_QW_K^TX^T)XW_VI_d)}{\partial W_V} \\
&= I_d \otimes \mathbf{softmax}(XW_QW_K^TX^T)X
\end{align}
which proves the final equality in \cref{thm:attn_derivs}. 
\end{proof}

\begin{proof}[Proof of \cref{thm:cond_jac}]
The proof of \cref{thm:cond_jac} follows from using \cref{thm:attn_derivs} and the definition of the Jacobian of $\mathbf{A}(X)$ with respect to $W_Q$, $W_K$ and $W_V$ given by
\begin{equation}
 J(\mathbf{A}(X)) =  
 \bigg{[}\frac{\partial{(\mathbf{A}(X))}}{\partial{W_Q}}, 
 \frac{\partial{(\mathbf{A}(X))}}{\partial{W_K}}, 
 \frac{\partial{(\mathbf{A}(X))}}{\partial{W_V}}\bigg{]}^T.  
\end{equation}
We recall that the condition number is defined as 
$\kappa(J(\mathbf{A}(X))) = \frac{\sigma_{\max}(J(\mathbf{A}(X)))}{\sigma_{\min}(J(\mathbf{A}(X)))}$ where 
$\sigma_{\max}(J(\mathbf{A}(X)))$ is the maximum singular value of $J(\mathbf{A}(X))$ and 
$\sigma_{\min}(J(\mathbf{A}(X)))$ the minimum singular value which we know is non-zero because of the assumption that $J(\mathbf{A}(X))$ has full rank.
Note that, using the notation in \cref{subsec:prelims}, we have that
$J(\mathbf{A}(X)) \in \R^{3dN \times dD}$ as 
$\frac{\partial\mathbf{A}(X)}{\partial W_Q}$, $\frac{\partial \mathbf{A}(X)}{\partial W_K}$, $\frac{\partial\mathbf{A}(X)}{\partial W_V} \in \R^{dN \times dD}$. For each of notation we will write 
$\mathbf{A}_Q := \frac{\partial\mathbf{A}(X)}{\partial W_Q}$, 
$\mathbf{A}_K := \frac{\partial \mathbf{A}(X)}{\partial W_K}$, 
$\mathbf{A}_V := \frac{\partial\mathbf{A}(X)}{\partial W_V}$.

We will start by computing a bound for the maximum singular value. We have for any vector $z \in \R^{dD}$ we have
\begin{align}
||J(\mathbf{A}(X))z||_2^2 &= 
 \left\lVert\bigg{[}\frac{\partial{(\mathbf{A}(X))}}{\partial{W_Q}}(z), 
 \frac{\partial{(\mathbf{A}(X))}}{\partial{W_K}}(z), 
 \frac{\partial{(\mathbf{A}(X))}}{\partial{W_V}}(z)\bigg{]}^T\right\rVert_2^2 \\
&=  \left\lVert\big{[}\mathbf{A}_Q(z), 
 \mathbf{A}_K(z), 
 \mathbf{A}_V(z)\big{]}\right\rVert_2^2 \\
&= 
\left\lVert \mathbf{A}_Q(z)\right\rVert_2^2 + 
\left\lVert \mathbf{A}_K(z)\right\rVert_2^2 +
\left\lVert \mathbf{A}_V(z)\right\rVert_2^2 \\
&\leq 
\left(\sigma_{\max}(\mathbf{A}_Q)^2 + \sigma_{\max}(\mathbf{A}_K)^2 + 
\sigma_{\max}(\mathbf{A}_V)^2\right)||z||_2^2. 
\end{align}
This implies that
\begin{align}
\sigma_{\max}(J(\mathbf{A}(X))) &:= \max_{z\neq 0}\frac{||J(\mathbf{A}(X))z||_2^2}{||z||_2^2} \\   
&\leq
\sqrt{\sigma_{\max}(\mathbf{A}_Q)^2 + \sigma_{\max}(\mathbf{A}_K)^2 + 
\sigma_{\max}(\mathbf{A}_V)^2} \\
&\leq \sigma_{\max}(\mathbf{A}_Q) + \sigma_{\max}(\mathbf{A}_K) + 
\sigma_{\max}(\mathbf{A}_V).
\end{align}
The next step is to compute a lower bound for the minimum singular value $\sigma_{\min}(J(\mathbf{A}(X)))$. The approach is similar to the above, using the fact that 
$\sigma_{\min}(J(\mathbf{A}(X))) = 
\min_{||z||_2 = 1}J(\mathbf{A}(X)(z)$. We can then use the inequality
\begin{align}
\left\lVert J(\mathbf{A}(X)(z)\right\rVert_2^2 &= 
\left\lVert
\bigg{[}\frac{\partial{(\mathbf{A}(X))}}{\partial{W_Q}}(z), 
 \frac{\partial{(\mathbf{A}(X))}}{\partial{W_K}}(z), 
 \frac{\partial{(\mathbf{A}(X))}}{\partial{W_V}}(z)\bigg{]}^T\right\rVert_2^2 \\
&\geq
\max\left\{\left\lVert\frac{\partial{(\mathbf{A}(X))}}{\partial{W_Q}}(z) \right\rVert,
\left\lVert\frac{\partial{(\mathbf{A}(X))}}{\partial{W_K}}(z) \right\rVert,
\left\lVert\frac{\partial{(\mathbf{A}(X))}}{\partial{W_V}}(z) \right\rVert
\right\} 
\end{align}
Then minimizing the above over the constraint $z \in \R^{dD}$ such that $||z|| = 1$ we obtain
\begin{align}
\sigma_{\min}(J(\mathbf{A}(X))) \geq 
\max\left\{
\sigma_{\min}\left(\frac{\partial{(\mathbf{A}(X))}}{\partial{W_Q}}\right), 
\sigma_{\min}\left(\frac{\partial{(\mathbf{A}(X))}}{\partial{W_K}}\right), 
\sigma_{\min}\left(\frac{\partial{(\mathbf{A}(X))}}{\partial{W_V}}\right)
\right\}.
\end{align}
Combing the bounds on $\sigma_{\max}(J(\mathbf{A}(X)))$ and 
$\sigma_{\min}(J(\mathbf{A}(X)))$ we obtain
\begin{align}
\kappa(J(\mathbf{A}(X))) \leq \kappa\left(\frac{\partial{(\mathbf{A}(X))}}{\partial{W_Q}}\right) + 
\kappa\left(\frac{\partial{(\mathbf{A}(X))}}{\partial{W_K}}\right)
+
\kappa\left(\frac{\partial{(\mathbf{A}(X))}}{\partial{W_V}}\right).
\end{align}
The final step is to get a bound on the condition numbers for each term on the right hand side in the above inequality. This is done using two facts: Firstly, given two matrices $C$ and $D$ such that 
the product $CD$ is full rank then $\kappa(CD) \leq \kappa(C)\kappa(D)$ and the second that $\kappa(C \otimes D) = 
\kappa(C)\kappa(D)$. Using these two facts we can then use 
\cref{thm:attn_derivs} to obtain the bound of \cref{thm:cond_jac} and the proof is finished.
\end{proof}

We will now give the proof of \cref{thm:attention_weights_regularize}.

\begin{proof}[Proof of \cref{thm:attention_weights_regularize}]
We will give the proof for $W_Q$ as the proof for $W_K$ and $W_V$ are exactly analogous.

Take the SVD of $W_Q$ to obtain $W_Q = USV^T$ where 
$U \in \R^{D\times D}$ is the left singular vectors, 
$S \in \R^{D \times d}$ is the matrix of singular values having the singular values of $W_Q$ on its main diagonal, $V \in \R^{d\times d}$ is the right singular vectors. We then define a new matrix $C_Q$ by the formula
\begin{equation}
    C_Q := U\overline{S}V^T
\end{equation}
where $\overline{S} \in \R^{D \times d}$ has $\sigma_{\max}(W_Q)$
on its main diagonal and zeros elsewhere. We then have that
\begin{equation}
W_Q + C_Q = U(S + \overline{S})V^T.
\end{equation}
We can then compute $\sigma_{\max}(W_Q +  C_Q) = 2\sigma_{\max}(W_Q)$ by definition of $C_Q$ and 
$\sigma_{\min}(W_Q +  C_Q) = \sigma_{\min}(W_Q) + \sigma_{\max}(W_Q)$.
This implies
\begin{align}
\kappa(W_Q +  C_Q) &= 
\frac{\sigma_{\max}(W_Q + C_Q)}{\sigma_{\min}(W_Q+C_Q)} \\
&=\frac{2\sigma_{\max}(W_Q)}{\sigma_{\min}(W_Q) + \sigma_{\max}(W_Q)} \\
&<
2
\end{align}
which finishes the proof of the theorem.
\end{proof}

A key issue for implementing the approach in \cref{thm:attention_weights_regularize} into the attention layer is that it requires the computation of the SVD of each of the matrices $W_Q$, $W_K$ and $W_V$. Computing the SVD each time for a large matrix is extremely memory intensive. This is exactly the premise of 
\cref{thm:imp_friendly} whose proof we now give.

\begin{proof}[Proof of \cref{thm:imp_friendly}]
 The proof starts by making use of the Weyl inequalities for singular values \cite{franklin2000matrix} to obtain
\begin{align}
    \sigma_{\max}(A + \lambda\cdot I_k) &\leq \sigma_{\max}(A) + 
    \sigma_{\max}(\lambda\cdot I_k). \label{eq:weyl_1}\\
    \sigma_{\min}(A + \lambda\cdot I_k) &\geq \sigma_{\min}(\lambda\cdot I_k) - \sigma_{\max}(A). \label{eq:weyl_2}
\end{align}
We then compute
\begin{align}
\kappa(A + \lambda\cdot I_k) &:= \frac{\sigma_{\max}(A + \lambda\cdot I_k)}{\sigma_{\min}(A + \lambda\cdot I_k)} \\
&\leq 
\frac{\sigma_{\max}(A) + \sigma_{\max}(\lambda\cdot I_k)}{\sigma_{\min}(\lambda\cdot I_k) - \sigma_{\min}(A)} \label{eqn:using_wely1}\\
&= \frac{\sigma_{\max}(A) + \lambda}{\lambda - \sigma_{\min}(A)} \\
&\leq \frac{\sigma_{\max}(A)}{\sigma_{\min}(A)} \text{ by our assumption } \\
&= \kappa(A)
\end{align}
which completes the proof.
\end{proof}

\paragraph{Other forms of Attention.} The theory developed in \cref{sec:theory} was presented for the self-attention mechanism defined in \cref{subsec:prelims}, to allow for concrete analytical expressions. However, as evident from the derivations, the proposed corrections to the weight matrices $W_Q$, $W_K$, and $W_V$ apply to any attention layer. Since the results depend only on the derivatives of the attention function with respect to these matrices, analogous theorems can be established for more general mechanisms such as cross-attention. Moreover, in \cref{sec:exps}, we empirically demonstrate that our theoretical results extend to a broad class of attention architectures.


\subsection{Experimental Analysis}\label{app:exps}

\paragraph{Details of implementation of spectral conditioned attention.} For all of the experiments in \cref{sec:exps} we implemented spectral conditioned attention following the strategy outlined in \cref{subsec:spec_cond} via \cref{thm:imp_friendly}. We describe this implementation in detail. For each attention layer in a transformer fix the query, key and value weights as $W_Q$, $W_K$, $W_V$. We then form three matrices denoted by $C_Q$, $C_K$ and $C_V$ with each being defined by the equation
\begin{align}
C_Q &= \lambda I_Q    \\
C_K &= \lambda I_K \\
C_V &= \lambda I_V
\end{align}
with $\lambda = 10$ and
where $C_Q$, $C_K$ and $C_V$ have the same shape as $W_Q$, $W_K$ and $W_V$ and $I_Q$, $I_K$ and $I_V$ has the same shape as $C_Q$, $C_K$ and $C_V$ with $1$'s on their main diagonal and zeros elsewhere. The matrices $C_Q$, $C_K$ and $C_V$, which we call the correction matrices, are fixed throughout training the transformer and never change. They are not updated and hence do not add any memory to the backpropagation algorithm and the only memory comes from storing them. 
During the training of the transformer, we add each of these correction matrices to the weights 
$W_Q$, $W_K$ and $W_V$ during each forward pass forming. Thus after initializing the weights $W^0_Q$, $W^0_K$ and $W^0_V$ we add the correction terms $C_Q$, $C_K$ and $C_V$ as follows
\begin{align}
&1. W^0_Q \rightarrow W^0_Q + C_Q \\
&2. W^0_K \rightarrow W^0_K + C_K \\
&3. W^0_V \rightarrow W^0_V + C_V
\end{align}
perform a forward pass using $W^0_Q + C_Q$, $W^0_K + C_K$ and 
$W^0_V + C_V$, then perform a backward pass and during the backward pass only the weights $W^0_Q$, $W^0_K$ and $W^0_V$ are updated forming $W^1_Q$, $W^1_K$ and $W^1_V$. 
This process then repeats, we add the fixed correction terms to $W^1_Q$, $W^1_K$ and $W^1_V$ forming 
\begin{align}
&1. W^1_Q \rightarrow W^1_Q + C_Q \\
&2. W^1_K \rightarrow W^1_K + C_K \\
&3. W^1_V \rightarrow W^1_V + C_V
\end{align}
then a backward pass updates $W^1_Q$, $W^1_K$ and $W^1_V$ and this process continues. Note that \cref{thm:imp_friendly} allows any $\lambda \geq 2$ when forming  Through ablations we found the best value is 
$\lambda = 10$.

\subsubsection{Vision transformers on Imagenet-1k}\label{app:ic}

\paragraph{Validating the theory.} In \cref{subsec:IC} we validated the theory given in \cref{sec:theory} on a ViT-B and XCiT-M architecture. For each experiment we ran five trials with five different random seeds and plotted the mean. The plots in \cref{subsec:IC} were shown in a log scale where the standard deviations were not visible. We have plotted each plot in a new scale that clearly shows the standard deviations. The plots for ViT-B can be seen in 
\cref{fig:vit_min_sing_std,fig:vit_max_sing_std,fig:vit_query_cond_std,fig:vit_cond_std} and the plots for XCiT-M can be seen from 
\cref{fig:xcit_min_sing_std,fig:xcit_max_sing_std,fig:xcit_query_cond_std,fig:xcit_cond_std}. Furthermore, \cref{fig:vit_max_sing_std} validates the assumption in \cref{thm:imp_friendly} on the maximum singular value for ViT-B and  \cref{fig:xcit_max_sing_std} for the maximum singular for XCiT justifying the use of \cref{thm:imp_friendly}.

\paragraph{Ablation on $\lambda$ from \cref{thm:imp_friendly}.} In \cref{sec:exps} we used the implementation for the correction terms given in \cref{sec:theory}. This required choosing a factor 
$\lambda \geq 2$. For the experiments in \cref{sec:exps} we chose $\lambda$ through a grid search treating it as a hyperparameter. An ablation for $\lambda$ on the ViT-B architecture trained on ImageNet-1k is shown in \cref{tab:lamb_vits}. What we found was that $\lambda \geq 10$ gives the best result.

\begin{table}[h]
\caption{Ablation on $\lambda$ for ViT-B}
\centering
\small
\setlength{\tabcolsep}{4pt}
\renewcommand{\arraystretch}{1.2}
\begin{tabular}{c|cccccccc}
\toprule
$\lambda$ & $2$ & $4$ & $6$ & $8$ & $10$ & $12$ & $14$ & $16$ \\
\midrule
ViT-B Acc. & 80.8 & 81.1 & 81.4 & 81.6 & 81.7 & 81.7 & 81.6 & 81.5 \\
\bottomrule
\end{tabular}
\label{tab:lamb_vits}
\end{table}


\paragraph{Relation to Normalization.} Our conditioning methodology can be seen as being related to normalization. By adding correction terms to the weights $W_Q$, $W_K$ and $W_V$ to improve their spectrum we are in effect producing a normalization for their spectrum. We therefore decided to test our spectral conditioning when we remove layer normalization in the vision transformers to understand whether spectral conditioning can replace layer normalization. \cref{tab:spec_cond_ab_norm} shows the results on the ViTs on ImageNet-1k. We ran each experiments 5 times and plot the mean and standard deviation in brackets. We found that if we remove layer normalization and keep spectral conditioning (with $\lambda = 10$) the performance drops. We noticed that this was due to the fact that the weight values in the corrected terms $W_Q + C_Q$, $W_K + C_K$ and $W_V + C_V$ were much larger than without layer normalization which can lead to issues with backpropagation and suggesting that layer normalization was still crucial to maintain weights from getting too large.

\begin{table}[h]
 \caption{Performance comparison of spectral conditioning with and without layer normalization.}
    \centering
    \small 
    \setlength{\tabcolsep}{3pt} 
    \begin{tabular}{cc c c c c}
        \toprule
        & ViT-B & DeiT-B & Swin-B & XCiT-M & DaViT-B \\
        \midrule
        Original (with only layer norm.) & 80.7 ($\pm$0.41) & 81.6 ($\pm$0.30) & 83.4 ($\pm$0.28) & 82.6 ($\pm$0.39) & 84.3 ($\pm$0.26) \\
        Spec. cond. + layer norm. & 81.7 ($\pm$0.38) & 82.6 ($\pm$0.32) & 84.1 ($\pm$0.25) & 83.5 ($\pm$0.35) & 84.9 ($\pm$0.21) \\
        Spec. cond. - layer norm. & 79.8 ($\pm$0.31) & 80.4 ($\pm$0.33) & 80.5 ($\pm$0.23) & 80.0 ($\pm$0.35) & 80.3 ($\pm$0.21) \\
        \bottomrule
    \end{tabular}
    \label{tab:spec_cond_ab_norm}
\end{table}

\paragraph{Comparing \cref{thm:attention_weights_regularize} and \cref{thm:imp_friendly}.} In \cref{sec:theory}, we saw that \cref{thm:attention_weights_regularize} gave a methodology that could reduce the condition number of the query weights, key weights and value weights to less than $2$. However, the reason we did not use this theorem as the basis of our experiments is that it requires an SVD computation which would significantly increase training times when compared to the implementation friendly theorem 
\cref{thm:imp_friendly}. We carried out a comparison on the vision transformer for both methods. As can be seen from \cref{tab:spectral_conditioning_results} the vision transformer employing the methodology from \cref{thm:attention_weights_regularize} take much longer to train.

\begin{table}[h]
\caption{Comparison of accuracy and training time across transformer variants with spectral conditioning of attention following \cref{thm:attention_weights_regularize} and \cref{thm:imp_friendly}.}
\centering
\small
\setlength{\tabcolsep}{6pt}
\renewcommand{\arraystretch}{1.1}
\begin{tabular}{lcc}
\toprule
\textbf{Model} & \textbf{Acc.} & \textbf{Training time (hrs:mins)} \\
\midrule
ViT-B (original) & 80.7 & 29:29 \\
ViT-B spec. cond. (Thm.~3.5) & 82.0 & 41:38 \\
ViT-B spec. cond. (Thm.~3.8) & 81.7 & 29:33 \\
DeiT-B (original) & 81.6 & 26:16 \\
DeiT-B spec. cond. (Thm.~3.5) & 82.9 & 37:54 \\
DeiT-B spec. cond. (Thm.~3.8) & 82.6 & 26:24 \\
Swin-B (original) & 83.4 & 53:12 \\
Swin-B spec. cond. (Thm.~3.5) & 84.3 & 68:32 \\
Swin-B spec. cond. (Thm.~3.8) & 84.1 & 53:26 \\
XCiT-M (original) & 82.6 & 91:03 \\
XCiT-M spec. cond. (Thm.~3.5) & 83.9 & 109:12 \\
XCiT-M spec. cond. (Thm.~3.8) & 83.5 & 91:18 \\
DaViT-M (original) & 84.3 & 75:09 \\
DaViT-M spec. cond. (Thm.~3.5) & 85.2 & 93:12 \\
DaViT-M spec. cond. (Thm.~3.8) & 84.9 & 75:28 \\
\bottomrule
\end{tabular}
\label{tab:spectral_conditioning_results}
\end{table}

\paragraph{Hardware and implementation.} The image classification experiments in \cref{subsec:IC} of the paper were done on Nvidia A100 GPUs.
The implementation of the ViTs were all done using the Timm code base \cite{rw2019timm}. The architectures were all trained from scratch on the ImageNet-1k dataset using the AdamW optimizer following the hyperparameters used in the original papers \cite{dosovitskiy2020image,liu2021swin,ali2021xcit,touvron2021training,ding2022davit}.

\paragraph{Memory Cost.}
We analyze the memory overhead introduced by using spectrally conditioned weight matrices $W_Q + C_Q$, $W_K + C_K$, and $W_V + C_V$ in the forward pass for the ViT-B architecture. A similar analysis gives a similar profile for all the other ViT architectures we consider in \cref{subsec:IC}.
The correction matrices $C_Q$, $C_K$, and $C_V$ are fixed diagonal matrices with constant entries (value $10$) on the main diagonal, are non-trainable, and have gradients disabled (\texttt{requires\_grad = False}).

Computing $X(W_Q + C_Q)$ expands as:
\begin{equation}
X(W_Q + C_Q) = XW_Q + XC_Q,
\end{equation}
where $XW_Q$ is a standard matrix multiplication and $XC_Q$ corresponds to scaling each column of $X$ by $10$. This operation is efficiently implemented as a column-wise scaling, requiring no explicit storage of $C_Q$ beyond a scalar or a small constant vector.

\begin{itemize}
    \item \textbf{Parameters:} No new trainable parameters are introduced; $C_Q$, $C_K$, and $C_V$ are fixed and non-trainable.
    \item \textbf{Gradient Storage:} Since the correction matrices do not participate in backpropagation, no additional gradient memory is used.
    \item \textbf{Activation Memory:} The computation of $XC_Q$ involves element-wise scaling and summation, which reuses existing output tensors without requiring new memory allocations.
\end{itemize}

In practice, the additional memory overhead from spectral conditioning is negligible, as it involves no new trainable parameters, no gradient storage, and no extra activation tensors.

\begin{table}[h]
\centering
\caption{Comparison of memory overhead for spectral conditioning relative to standard attention projections.}
\label{tab:memory_cost}
\begin{tabular}{lcc}
\toprule
\textbf{Memory Aspect} & \textbf{Original Attention} & \textbf{With Spectral Conditioning} \\
\midrule
Trainable Parameters & $3 \times Dd$ & $3 \times Dd$ \\
Gradient Storage & $3 \times Dd$ & $3 \times Dd$ \\
Correction Matrices ($C_Q$, $C_K$, $C_V$) & -- & negligible (constant scalar) \\
Activation Memory & $3 \times N d$ & $3 \times N d$ (no extra tensors) \\
\bottomrule
\end{tabular}
\end{table}

\paragraph{FLOPS Cost.}
We analyze the additional FLOPS overhead introduced by using spectrally conditioned weight matrices $W_Q + C_Q$, $W_K + C_K$, and $W_V + C_V$ in the forward pass of the attention mechanism in the forward pass for the ViT-B architecture. A similar analysis gives a similar profile for all the other ViT architectures we consider in \cref{subsec:IC}. Consider the self-attention formulation:
\begin{equation}
\mathbf{softmax}(X (W_Q + C_Q)(W_K + C_K)^\top X^\top) \cdot X(W_V + C_V),
\end{equation}
where $X \in \mathbb{R}^{N \times D}$ is the input, and $W_Q$, $W_K$, $W_V \in \mathbb{R}^{D \times d}$ are the learnable projection matrices. The correction matrices $C_Q$, $C_K$, and $C_V$ are fixed diagonal matrices with constant entries $10$ along their main diagonal.

Computing the term $X(W_Q + C_Q)$ expands as:
\begin{equation}
X(W_Q + C_Q) = XW_Q + XC_Q.
\end{equation}
The first term, $XW_Q$, corresponds to a standard matrix multiplication with computational complexity $\mathcal{O}(N D d)$, requiring $2NDd$ FLOPS. The second term, $XC_Q$, corresponds to scaling each column of $X$ by $10$, which requires only $Nd$ FLOPS (one multiplication per entry).

Thus, the total FLOPS for computing $X(W_Q + C_Q)$ is:
\begin{equation}
2NDd + Nd.
\end{equation}
An identical cost arises for $X(W_K + C_K)$ and $X(W_V + C_V)$. Therefore, the additional FLOPS introduced by the correction terms across all three projections is:
\begin{equation}
3Nd.
\end{equation}
Compared to the original cost of $6NDd$ FLOPS for the query, key, and value projections without spectral conditioning, the overhead introduced by spectral conditioning is negligible:
\begin{equation}
\frac{3Nd}{6NDd} = \frac{1}{2D},
\end{equation}
which is insignificant for typical values of $D$ used in practice (e.g., $D = 512, 768, 1024$ are common dimension for ViTs).

\begin{table}[h]
\centering
\caption{FLOPS comparison for the forward pass of the attention projection layers with and without spectral conditioning. $N$ denotes the sequence length, $D$ the embedding dimension, and $d$ the projection dimension per head.}
\label{tab:flops_spectral_conditioning}
\begin{tabular}{lcc}
\toprule
\textbf{Operation} & \textbf{FLOPS (Original)} & \textbf{FLOPS (With Spectral Conditioning)} \\
\midrule
Query projection ($XW_Q$) & $2NDd$ & $2NDd + Nd$ \\
Key projection ($XW_K$)    & $2NDd$ & $2NDd + Nd$ \\
Value projection ($XW_V$)  & $2NDd$ & $2NDd + Nd$ \\
\midrule
\textbf{Total Projection FLOPS} & $6NDd$ & $6NDd + 3Nd$ \\
\bottomrule
\end{tabular}
\end{table}

\begin{figure}[ht!]
    \centering
    \includegraphics[width=1.\linewidth]
    {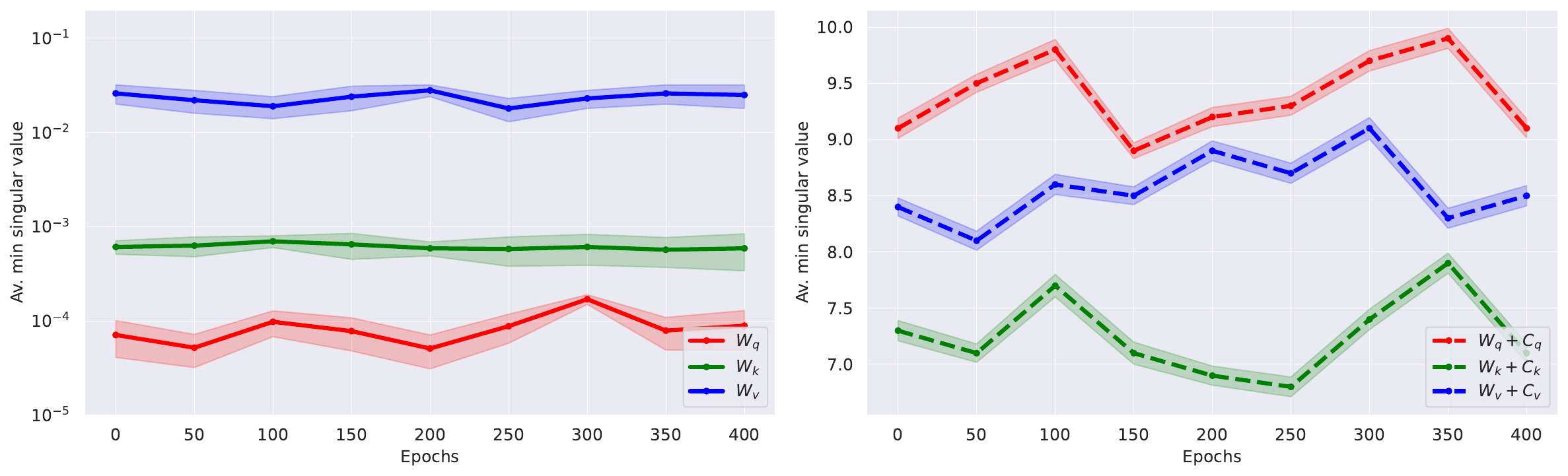}
    \caption{\textbf{Left:} Average minimum singular value of the query, key, and value projection matrices ($W_Q$, $W_K$, $W_V$) for a ViT-B during training. We plot the mean over five trials and the standard deviation.
    \textbf{Right:} Average minimum singular value of the corrected query, key, and value projection matrices ($W_Q + C_Q$, $W_K + C_K$, $W_V + C_V$) for a ViT-B during training. We plot the mean over five trials and the standard deviation.}
    \label{fig:vit_min_sing_std}
\end{figure}

\begin{figure}[ht!]
    \centering
    \includegraphics[width=1.\linewidth]
    {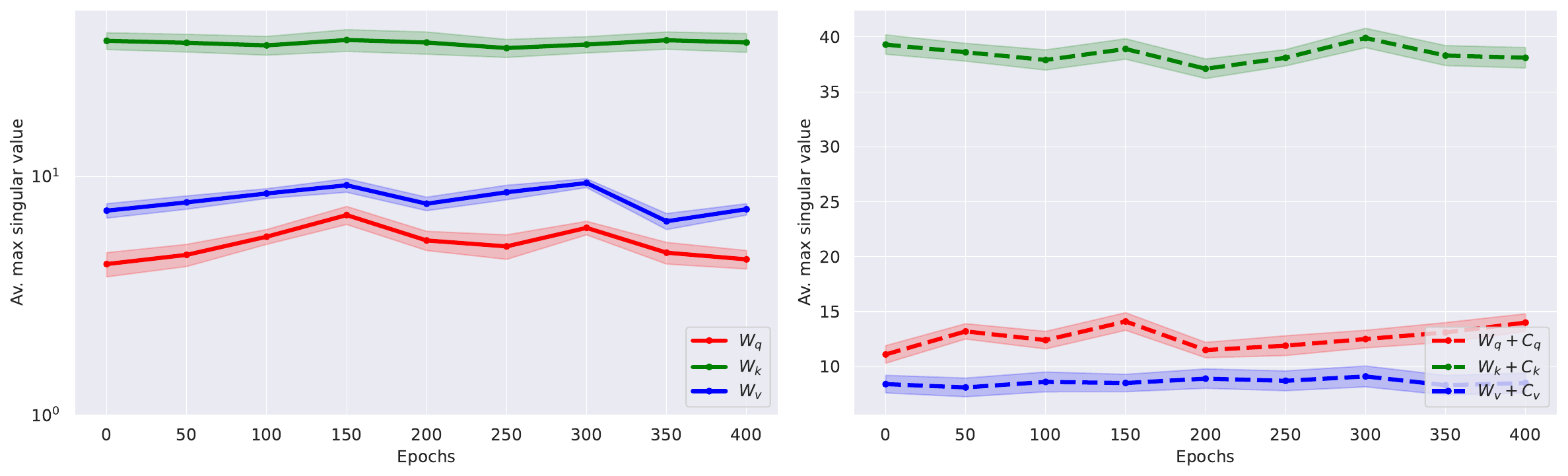}
    \caption{\textbf{Left:} Average maximum singular value of the query, key, and value projection matrices ($W_Q$, $W_K$, $W_V$) for a ViT-B during training. We plot the mean over five trials and the standard deviation.
    \textbf{Right:} Average maximum singular value of the corrected query, key, and value projection matrices ($W_Q + C_Q$, $W_K + C_K$, $W_V + C_V$) for a ViT-B during training. We plot the mean over five trials and the standard deviation.}
    \label{fig:vit_max_sing_std}
\end{figure}

\begin{figure}[ht!]
    \centering
    \includegraphics[width=1.\linewidth]
    {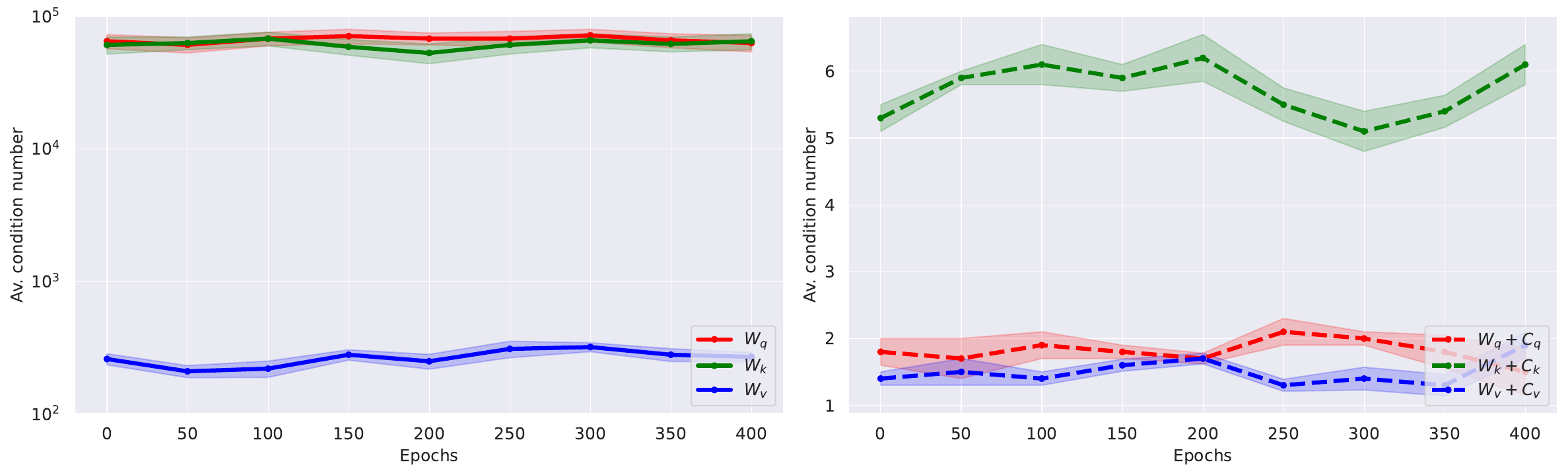}
    \caption{\textbf{Left:} Average condition number of the query, key, and value projection matrices ($W_Q$, $W_K$, $W_V$) for a ViT-B during training. We plot the mean over five trials and the standard deviation.
    \textbf{Right:} Average condition number of the corrected query, key, and value projection matrices ($W_Q + C_Q$, $W_K + C_K$, $W_V + C_V$) for a ViT-B during training. We plot the mean over five trials and the standard deviation.}
    \label{fig:vit_query_cond_std}
\end{figure}

\begin{figure}[ht!]
    \centering
    \includegraphics[width=0.7\linewidth]
    {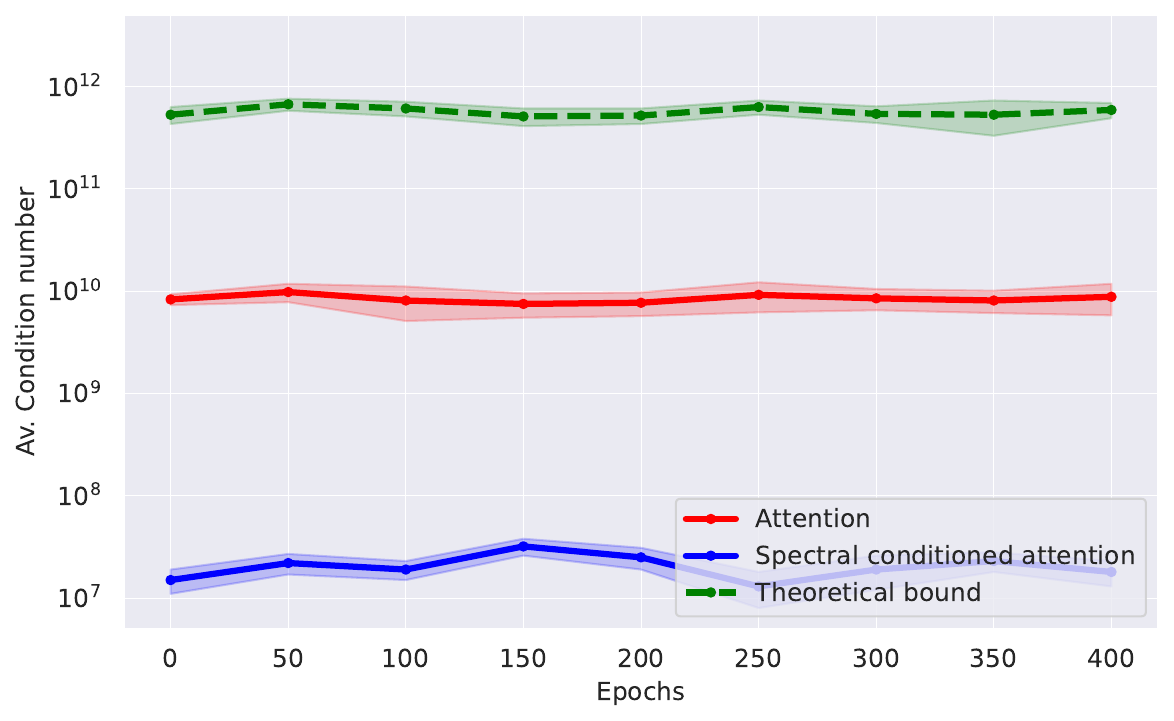}
    \caption{Average condition number of the self-attention Jacobian of a ViT-B over the course of training, before and after spectral conditioning, along with the theoretical bound from \cref{eqn:cond_jac}. We ran five trials with five different random seeds with plots showing the mean and standard deviations.}
    \label{fig:vit_cond_std}
\end{figure}

\begin{figure}[ht!]
    \centering
    \includegraphics[width=1.\linewidth]
    {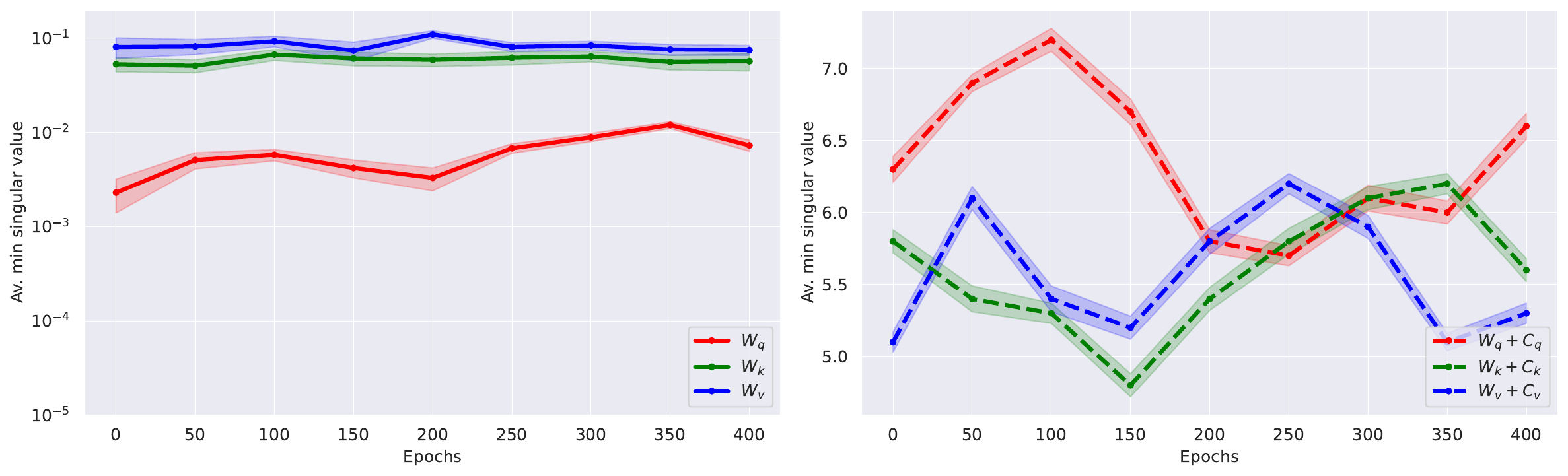}
    \caption{\textbf{Left:} Average minimum singular value of the query, key, and value projection matrices ($W_Q$, $W_K$, $W_V$) for a XCiT-M during training. We plot the mean over five trials and the standard deviation.
    \textbf{Right:} Average minimum singular value of the corrected query, key, and value projection matrices ($W_Q + C_Q$, $W_K + C_K$, $W_V + C_V$) for a XCiT-M during training. We plot the mean over five trials and the standard deviation.}
    \label{fig:xcit_min_sing_std}
\end{figure}

\begin{figure}[ht!]
    \centering
    \includegraphics[width=1.\linewidth]
    {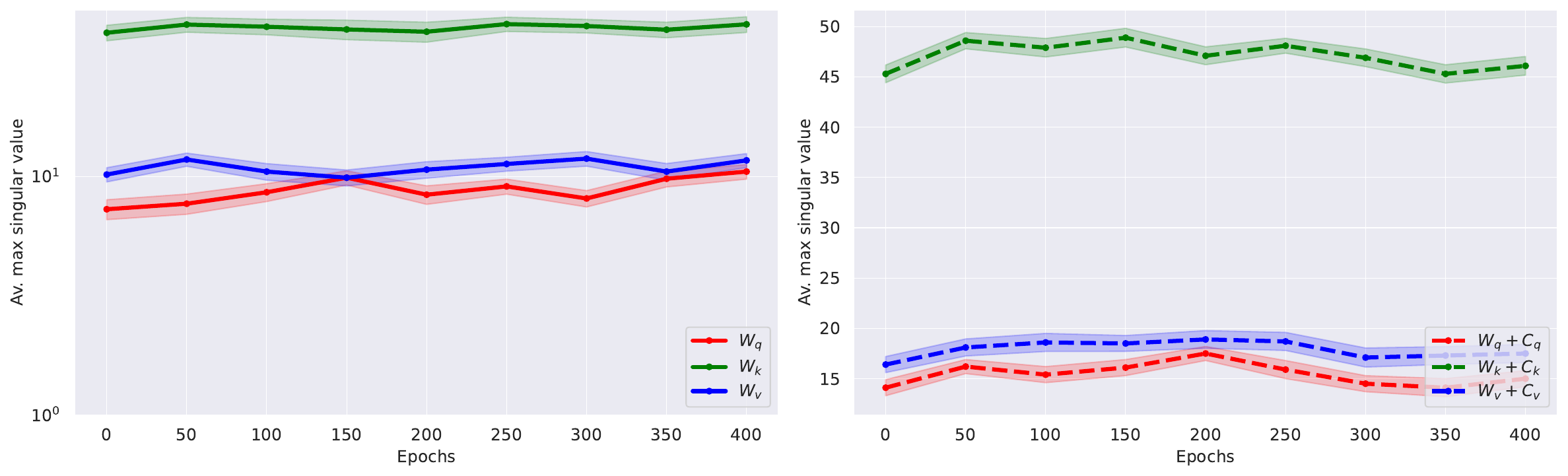}
    \caption{\textbf{Left:} Average maximum singular value of the query, key, and value projection matrices ($W_Q$, $W_K$, $W_V$) for a XCiT-M during training. We plot the mean over five trials and the standard deviation.
    \textbf{Right:} Average maximum singular value of the corrected query, key, and value projection matrices ($W_Q + C_Q$, $W_K + C_K$, $W_V + C_V$) for a XCiT-M during training. We plot the mean over five trials and the standard deviation.}
    \label{fig:xcit_max_sing_std}
\end{figure}

\begin{figure}[ht!]
    \centering
    \includegraphics[width=1.\linewidth]
    {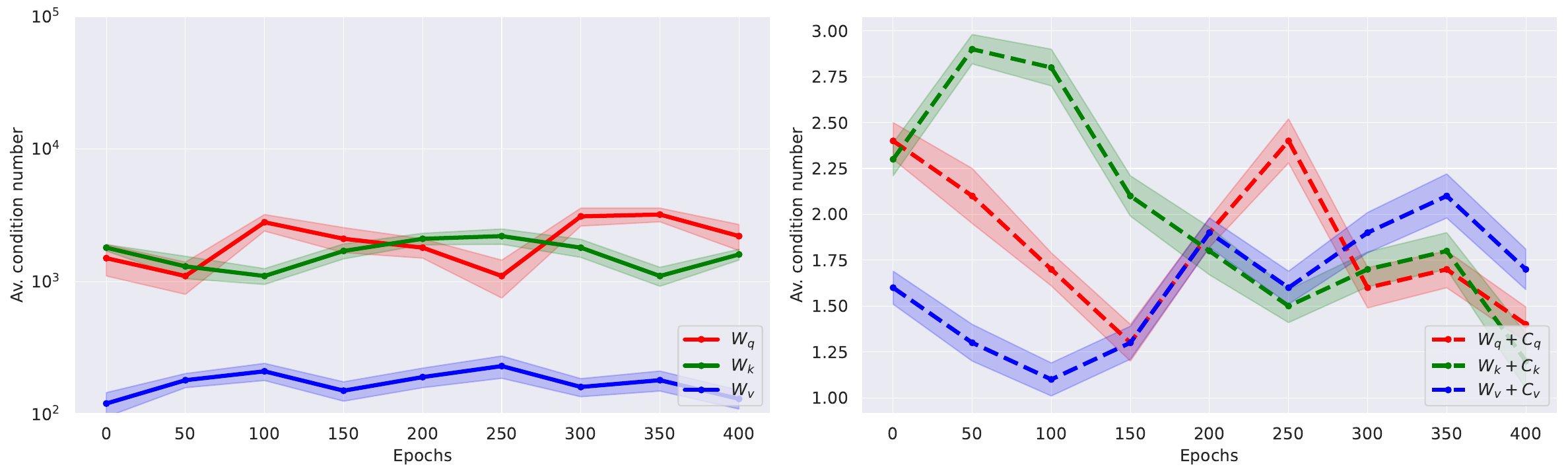}
    \caption{\textbf{Left:} Average condition number of the query, key, and value projection matrices ($W_Q$, $W_K$, $W_V$) for a XCiT-M during training. We plot the mean over five trials and the standard deviation.
    \textbf{Right:} Average condition number of the corrected query, key, and value projection matrices ($W_Q + C_Q$, $W_K + C_K$, $W_V + C_V$) for a XCiT-M during training. We plot the mean over five trials and the standard deviation.}
    \label{fig:xcit_query_cond_std}
\end{figure}

\begin{figure}[ht!]
    \centering
    \includegraphics[width=0.7\linewidth]
    {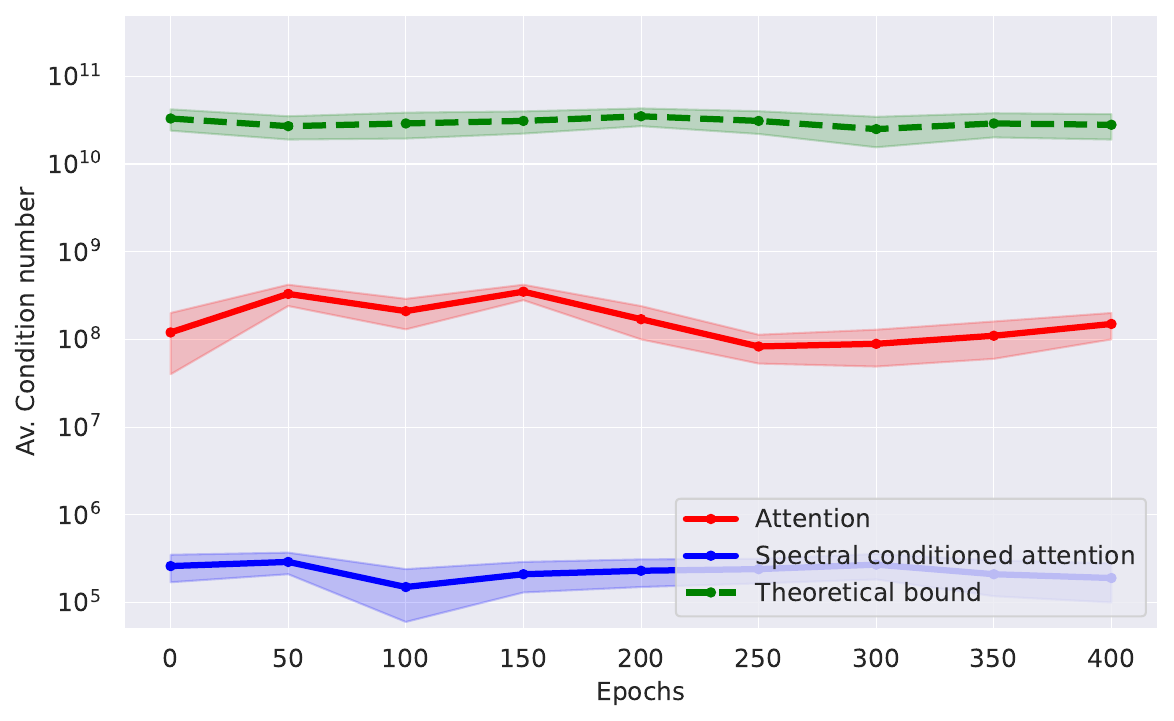}
    \caption{Average condition number of the self-attention Jacobian of a XCiT-M over the course of training, before and after spectral conditioning, along with the theoretical bound from \cref{eqn:cond_jac}. We ran five trials with five different random seeds with plots showing the mean and standard deviations.}
    \label{fig:xcit_cond_std}
\end{figure}

\subsubsection{Object detection and instance segmentation}\label{app:od}

\paragraph{Hardware and Implementation:} The experiments for \cref{subsec:OD} of the paper on object detection and instance segmentation were carried out on Nvidia A100 GPUs. The implementation followed 
 \cite{he2017mask}. We used the code base given by the GitHub \cite{matterport_maskrcnn} following their exact training regime.

 \paragraph{Memory and Overheads:} For this experiment we used a pre-trained XCiT-S and since, as explained in \cref{app:ic}, the correction terms are not trained and fixed throughout training there was no real memory overhead.

 \paragraph{Further results:} We also ran the object detection and instance segmentation results using a Swin-base vision transformer as the backbone. The result are shown in \cref{tab:swin_speccond_results}.

 \begin{table}[h]
\caption{Comparison of Swin-base and spectrally conditioned (spec. cond.) Swin-base on COCO benchmarks.}
    \centering
    \small
    \setlength{\tabcolsep}{6pt}
    \renewcommand{\arraystretch}{1.1}
    \begin{tabular}{lcccccc}
        \toprule
        \textbf{Model} & $\text{AP}^\text{b}$ & $\text{AP}^\text{b}_{50}$ & $\text{AP}^\text{b}_{75}$ & $\text{AP}^\text{m}$ & $\text{AP}^\text{m}_{50}$ & $\text{AP}^\text{m}_{75}$ \\
        \midrule
        Swin-base & 45.9 & 67.1 & 50.1 & 41.1 & 63.7 & 43.8 \\
        Spec. cond. Swin-base & \textbf{46.8} & \textbf{68.1} & \textbf{50.7} & \textbf{41.8} & \textbf{64.2} & \textbf{44.6} \\
        \bottomrule
    \end{tabular}
    \label{tab:swin_speccond_results}
\end{table}

 \subsection{Nystr\"omformer on LRA benchmark}\label{app:nyst}

 \paragraph{Validating the theory.} In \cref{subsec:nyst} we validated the theory given in \cref{sec:theory} on a 
 Nystr\"omformer on the text classification task in the LRA benchmark.
 For each experiment we ran five trials with five different random seeds and plotted the mean. The plots in \cref{subsec:nyst} were shown in a log scale where the standard deviations were not visible. We have plotted each plot in a new scale that clearly shows the standard deviations. The plots are shown from 
\cref{fig:nyst_text_min_sing_std,fig:nyst_text_max_sing_std,fig:nyst_text_cond_query_std,fig:nyst_text_cond_std}. We also ran a similar empirical analysis for the listops task in the LRA benchmark. The results of which can be seen in \cref{fig:nyst_listops_analysis,fig:nyst_listops_min_sing_std,fig:nyst_listops_max_sing_std,fig:nyst_listops_cond_query_std,fig:nyst_listops_cond_std}. Furthermore, \cref{fig:nyst_text_max_sing_std} and \cref{fig:nyst_listops_max_sing_std} validates the assumption in \cref{thm:imp_friendly} on the maximum singular value justifying the use of \cref{thm:imp_friendly}.

\paragraph{Ablation on $\lambda$ from \cref{thm:imp_friendly}.} For the experiments in \cref{subsec:nyst} we chose $\lambda$ through a grid search treating it as a hyperparameter. An ablation for $\lambda$ on the Nystr\"{o}mformer trained on the text classification task on the LRA benchmark is shown in \cref{tab:lamb_nyst}. Once again we found that $\lambda \geq 10$ gives the best result.

 \begin{table}[h]
 \caption{Ablation on $\lambda$ for Nystr\"{o}mformer on text classification task.}
\centering
\small
\setlength{\tabcolsep}{4pt}
\renewcommand{\arraystretch}{1.2}
\begin{tabular}{c|cccccccc}
\toprule
$\lambda$ & $2$ & $4$ & $6$ & $8$ & $10$ & $12$ & $14$ & $16$ \\
\midrule
Acc. & 63.7 & 64.1 & 64.3 & 64.6 & 64.8 & 64.8 & 64.7 & 64.7 \\
\bottomrule
\end{tabular}
\label{tab:lamb_nyst}
\end{table}


\paragraph{Relation to Normalization.} We compared spectral conditioning on a Nystr\"omformer with and without layer normalization. As can be seen in \cref{tab:spec_norm_nyst} we see that layer normalization is important for spectral conditioning. Once again when we removed layer normalization and just applied spectral conditioning we found that the weights were large resulting in poorer performance.

\begin{table}[h]
\caption{Comparing spectral conditioning with and without layer normalization for the Nystr\"omformer on text classification.}
    \centering
    \small 
    \setlength{\tabcolsep}{6pt} 
    \begin{tabular}{cc}
        \hline
         & Acc.(\%) \\
        \hline
        Original (with only layer norm.) & 63.8 ($\pm$0.24) \\
        \hline
        Spec. cond. + layer norm. & 64.8 ($\pm$0.20) \\
        \hline
        Spec. cond. - layer norm & 62.3 ($\pm$ 0.23) \\
        \hline
    \end{tabular}
    \label{tab:spec_norm_nyst}
\end{table}

\paragraph{Hardware.} All the experiments for the
Nystr\"omformer on LRA benchmark results in \cref{subsec:nyst} were carried out on Nvidia A100 GPUs following the implementation and hyperparameter settings given in \cite{nystromformer_github}. 

\paragraph{Memory and Overheads:} The analysis carried out on FLOPS and memory in \cref{app:ic} apply in this case as once again the correction matrices $C_Q$, $C_K$ and $C_V$ are fixed throughout training and only added for forward passes.

\begin{figure}[ht!]
    \centering
    \includegraphics[width=1.\linewidth]
    {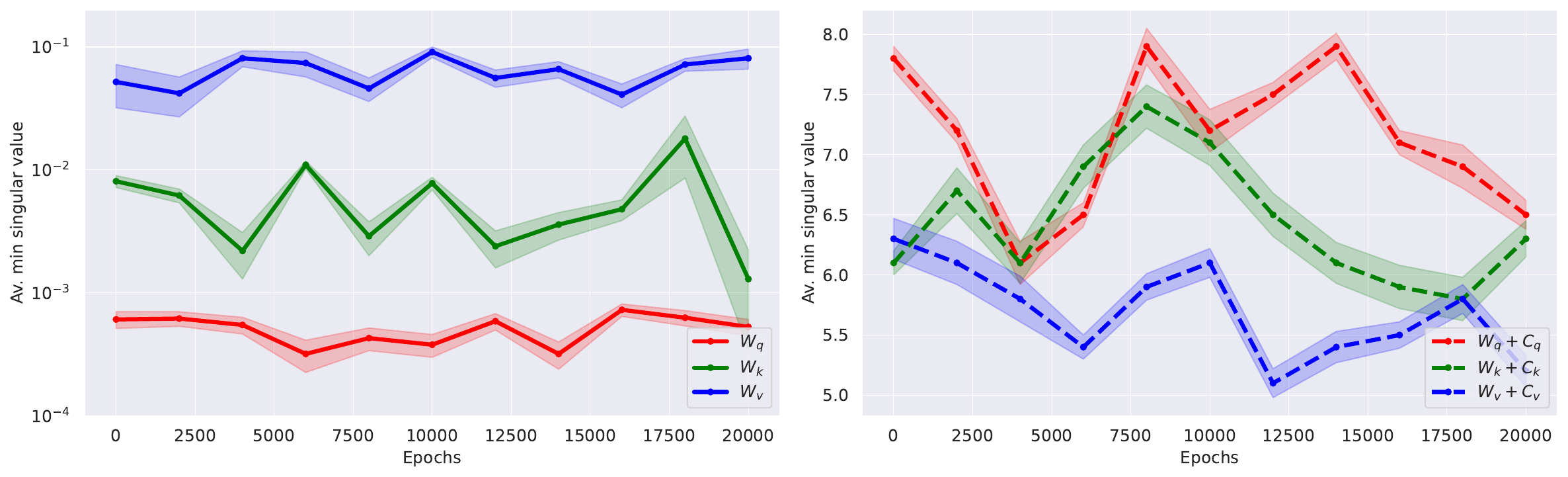}
    \caption{\textbf{Left:} Average minimum singular value of the query, key, and value projection matrices ($W_Q$, $W_K$, $W_V$) for a Nystr\"omformer on the text classification task during training. We plot the mean over five trials and the standard deviation.
    \textbf{Right:} Average minimum singular value of the corrected query, key, and value projection matrices ($W_Q + C_Q$, $W_K + C_K$, $W_V + C_V$) for a Nystr\"omformer on the text classification during training. We plot the mean over five trials and the standard deviation.}
    \label{fig:nyst_text_min_sing_std}
\end{figure}

\begin{figure}[ht!]
    \centering
    \includegraphics[width=1.\linewidth]
    {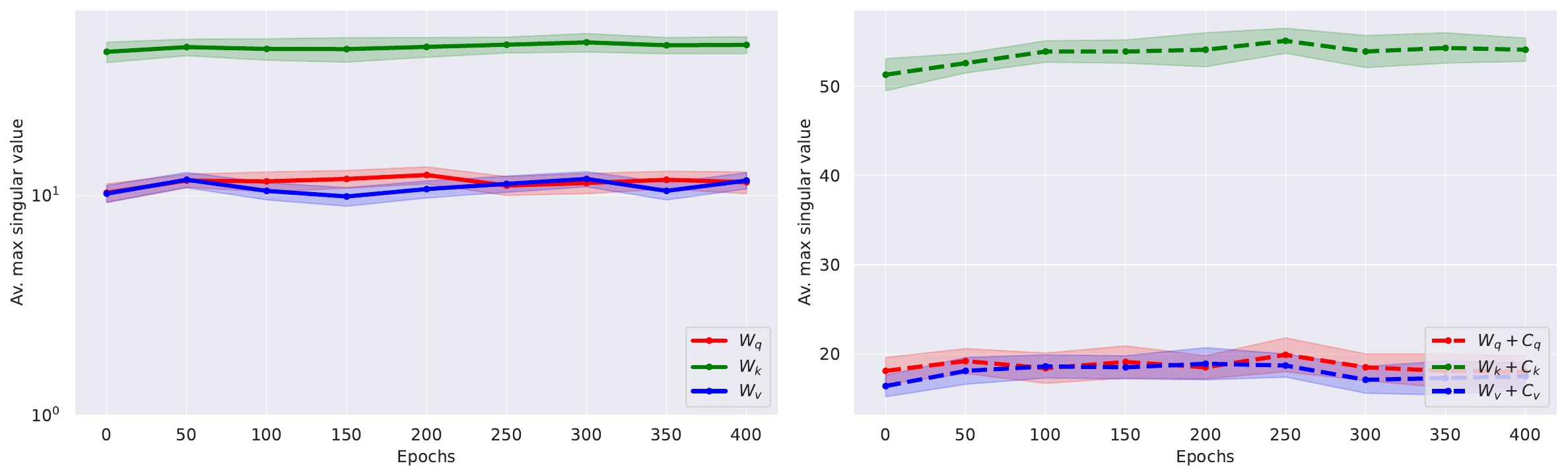}
    \caption{\textbf{Left:} Average maximum singular value of the query, key, and value projection matrices ($W_Q$, $W_K$, $W_V$) for a Nystr\"omformer on the text classification task during training. We plot the mean over five trials and the standard deviation.
    \textbf{Right:} Average maximum singular value of the corrected query, key, and value projection matrices ($W_Q + C_Q$, $W_K + C_K$, $W_V + C_V$) for a Nystr\"omformer on the text classification during training. We plot the mean over five trials and the standard deviation.}
    \label{fig:nyst_text_max_sing_std}
\end{figure}

\begin{figure}[ht!]
    \centering
    \includegraphics[width=1.\linewidth]
    {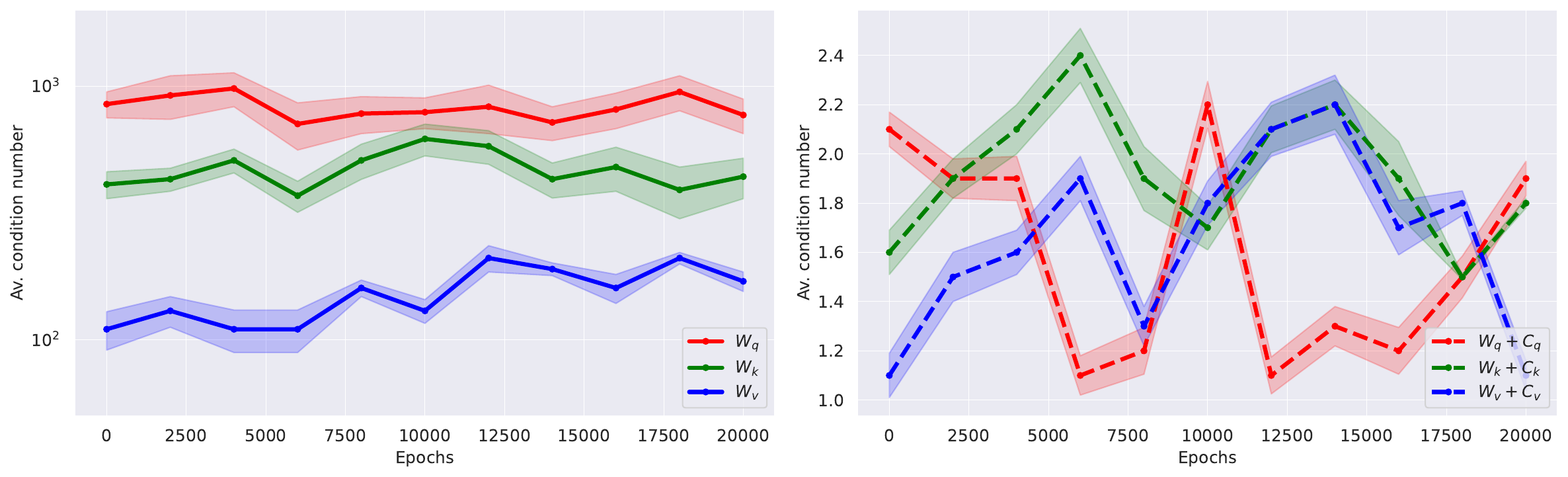}
    \caption{\textbf{Left:} Average condition number of the query, key, and value projection matrices ($W_Q$, $W_K$, $W_V$) for a Nystr\"omformer on the text classification task during training. We plot the mean over five trials and the standard deviation.
    \textbf{Right:} Average condition number of the corrected query, key, and value projection matrices ($W_Q + C_Q$, $W_K + C_K$, $W_V + C_V$) for a Nystr\"omformer on the text classification task during training. We plot the mean over five trials and the standard deviation.}
    \label{fig:nyst_text_cond_query_std}
\end{figure}

\begin{figure}[ht!]
    \centering
    \includegraphics[width=0.7\linewidth]
    {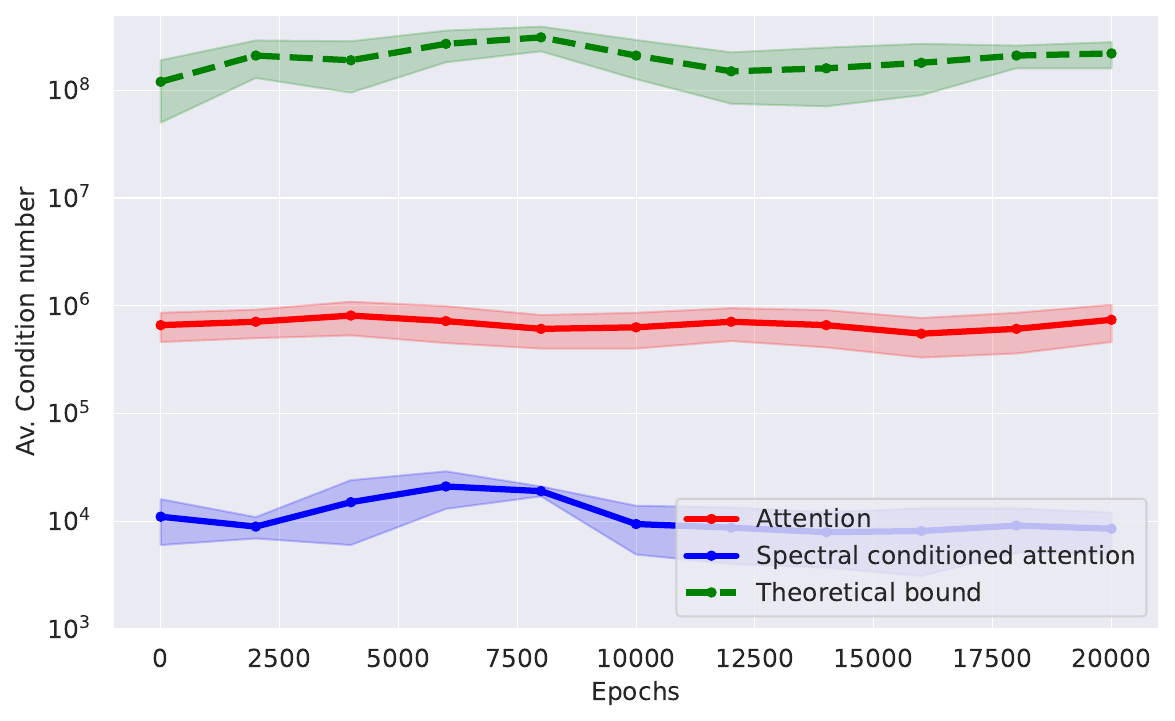}
    \caption{Average condition number of the self-attention Jacobian of a Nystr\"omformer on the text classification task over the course of training, before and after spectral conditioning, along with the theoretical bound from \cref{eqn:cond_jac}. We ran five trials with five different random seeds with plots showing the mean and standard deviations.}
    \label{fig:nyst_text_cond_std}
\end{figure}

\begin{figure}[ht!]
    \centering
    \includegraphics[width=1.\linewidth]
    {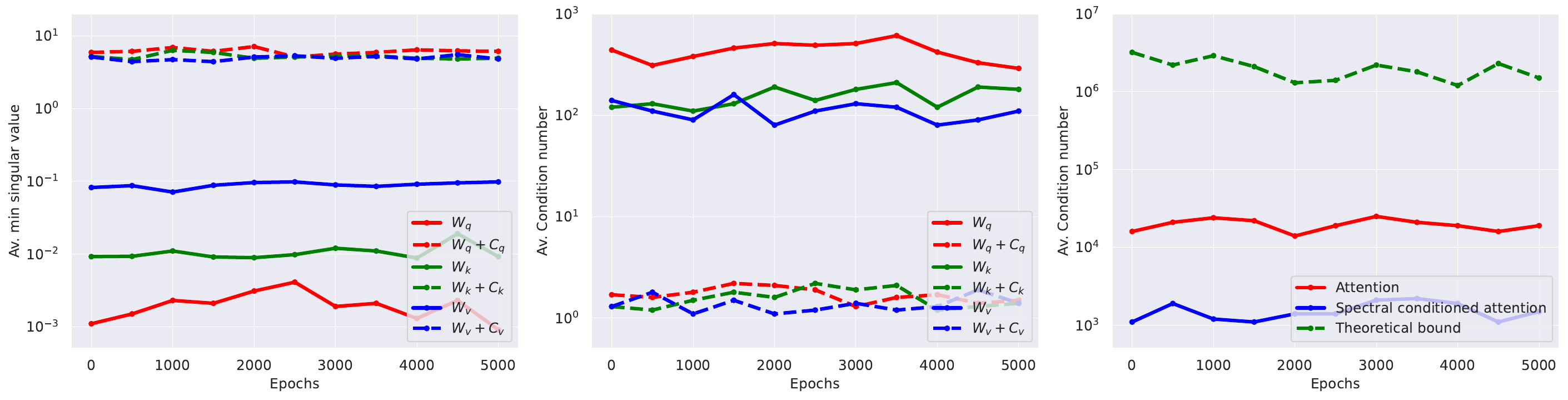}
    \caption{Analysis for Nystr\"{o}mformer on listops task. \textbf{Left:} Average minimum singular value of the query, key, and value projection matrices ($W_Q$, $W_K$, $W_V$) and their spectrally conditioned counterparts ($W_Q + C_Q$, $W_K + C_K$, $W_V + C_V$) throughout training. \textbf{Middle:} Condition numbers of $W_Q$, $W_K$, and $W_V$, and their spectrally conditioned forms during training. \textbf{Right:} Average condition number of the attention Jacobian over the course of training, before and after spectral conditioning, along with the theoretical bound from \cref{eqn:cond_jac}.}
    \label{fig:nyst_listops_analysis}
\end{figure}

\begin{figure}[ht!]
    \centering
    \includegraphics[width=1.\linewidth]
    {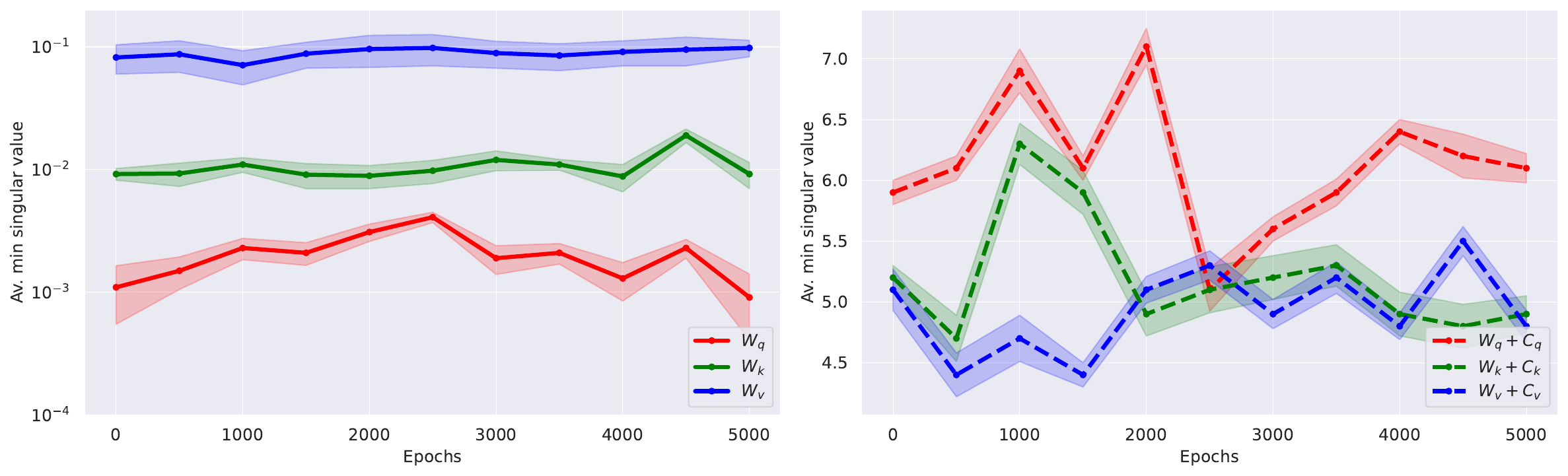}
    \caption{\textbf{Left:} Average minimum singular value of the query, key, and value projection matrices ($W_Q$, $W_K$, $W_V$) for a Nystr\"omformer on the listops task during training. We plot the mean over five trials and the standard deviation.
    \textbf{Right:} Average minimum singular value of the corrected query, key, and value projection matrices ($W_Q + C_Q$, $W_K + C_K$, $W_V + C_V$) for a Nystr\"omformer on the listops task during training. We plot the mean over five trials and the standard deviation.}
    \label{fig:nyst_listops_min_sing_std}
\end{figure}

\begin{figure}[ht!]
    \centering
    \includegraphics[width=1.\linewidth]
    {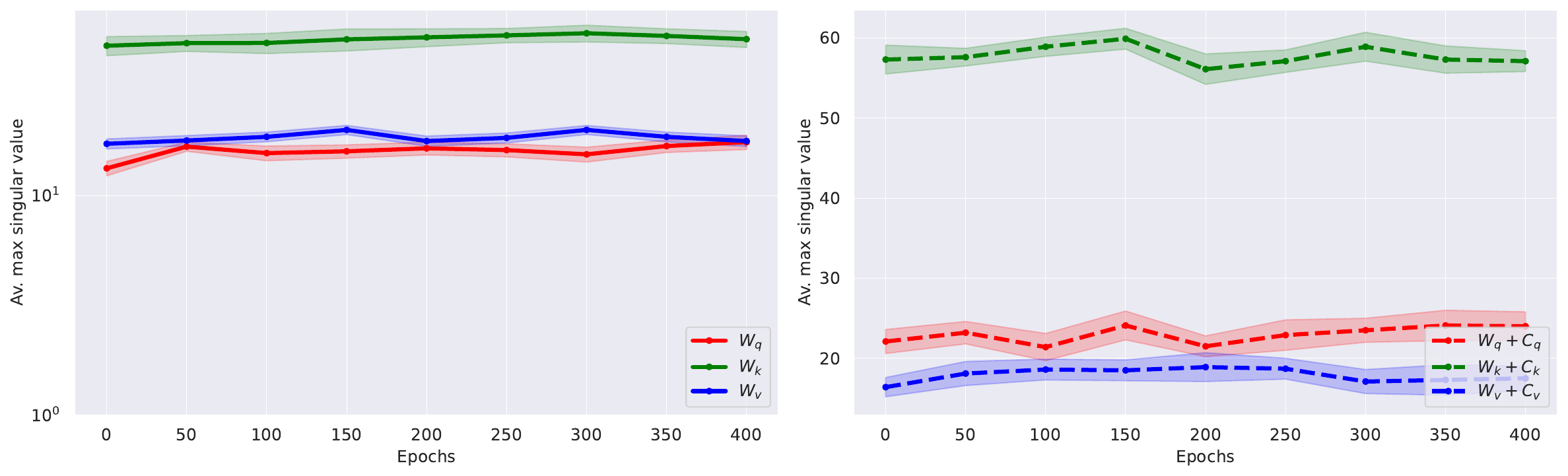}
    \caption{\textbf{Left:} Average maximum singular value of the query, key, and value projection matrices ($W_Q$, $W_K$, $W_V$) for a Nystr\"omformer on the listops task during training. We plot the mean over five trials and the standard deviation.
    \textbf{Right:} Average maximum singular value of the corrected query, key, and value projection matrices ($W_Q + C_Q$, $W_K + C_K$, $W_V + C_V$) for a Nystr\"omformer on the listops task during training. We plot the mean over five trials and the standard deviation.}
    \label{fig:nyst_listops_max_sing_std}
\end{figure}

\begin{figure}[ht!]
    \centering
    \includegraphics[width=1.\linewidth]
    {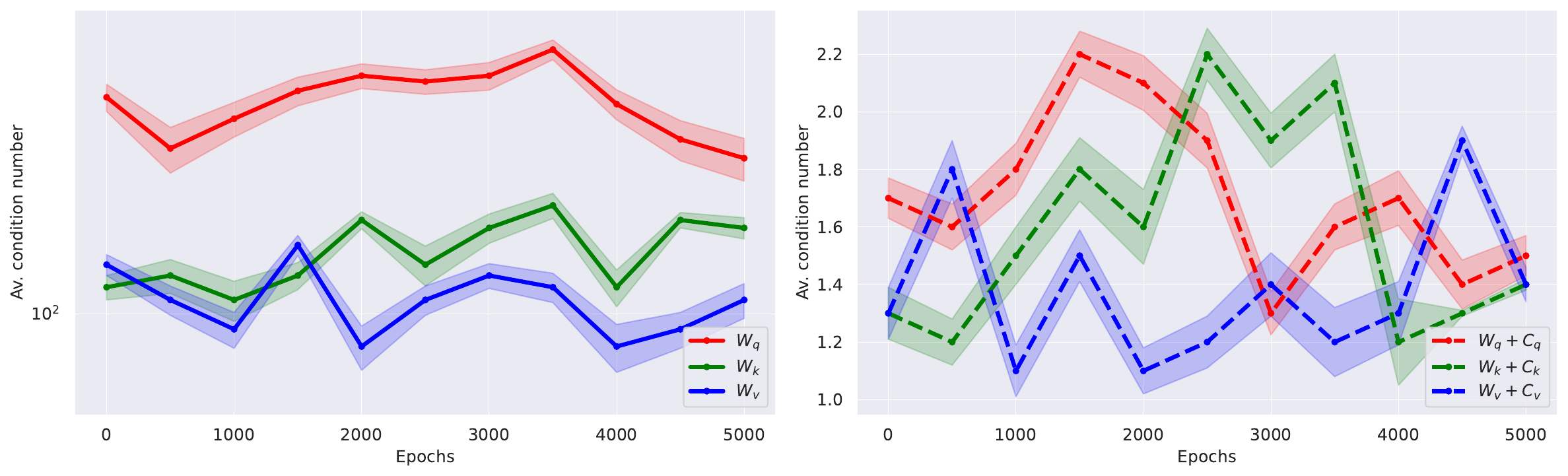}
    \caption{\textbf{Left:} Average condition number of the query, key, and value projection matrices ($W_Q$, $W_K$, $W_V$) for a Nystr\"omformer on the listops task during training. We plot the mean over five trials and the standard deviation.
    \textbf{Right:} Average condition number of the corrected query, key, and value projection matrices ($W_Q + C_Q$, $W_K + C_K$, $W_V + C_V$) for a Nystr\"omformer on the listops task during training. We plot the mean over five trials and the standard deviation.}
    \label{fig:nyst_listops_cond_query_std}
\end{figure}

\begin{figure}[ht!]
    \centering
    \includegraphics[width=0.7\linewidth]
    {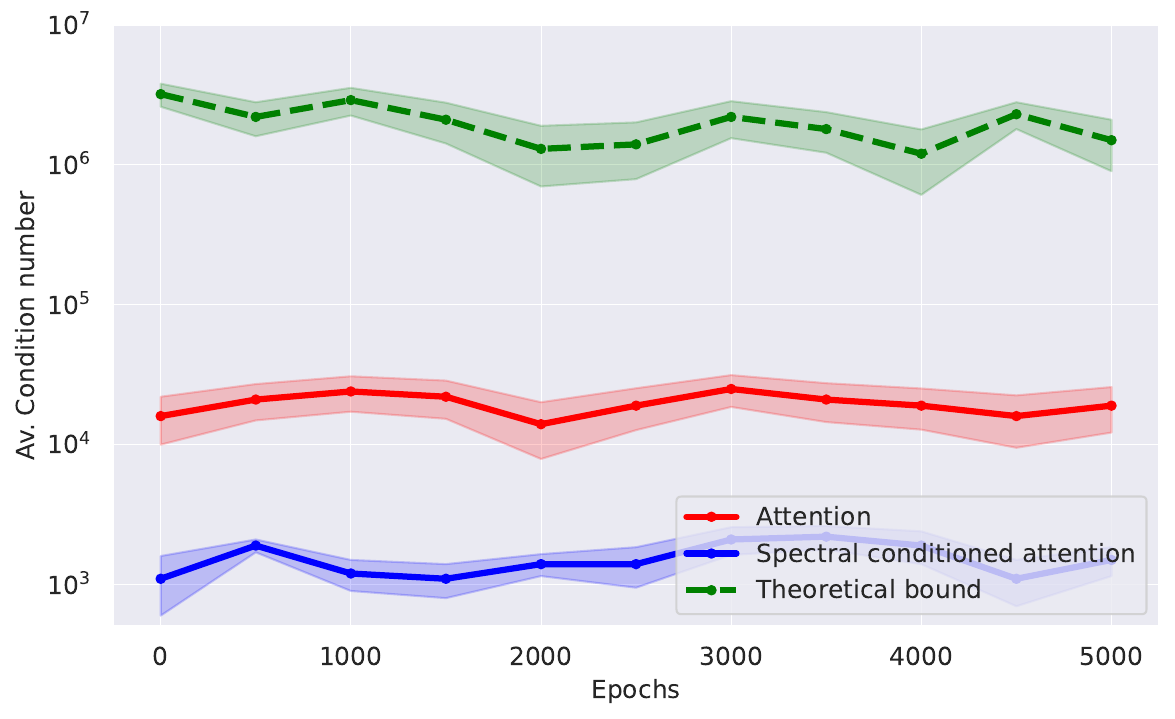}
    \caption{Average condition number of the self-attention Jacobian of a Nystr\"omformer on the listops task over the course of training, before and after spectral conditioning, along with the theoretical bound from \cref{eqn:cond_jac}. We ran five trials with five different random seeds with plots showing the mean and standard deviations.}
    \label{fig:nyst_listops_cond_std}
\end{figure}

\subsection{Language Modeling with Crammed BERT}\label{app:cram}

\paragraph{Hardware.} The language modeling experiment in \cref{subsec:lang} were all carried out on a Nvidia A6000 GPU. The Crammed-Bert was implemented following the original paper \cite{geiping2023cramming} and the original GitHub \cite{CrammingGit}. The training regime follows \cite{CrammingGit}.

\paragraph{Memory and Overheads:} The analysis carried out on FLOPS and memory in \cref{app:ic} apply in this case as once again the correction matrices $C_Q$, $C_K$ and $C_V$ are fixed throughout training and only added for forward passes.


\end{document}